%% file: neurips_2026.tex
\documentclass{article}

\PassOptionsToPackage{numbers,sort&compress}{natbib}

\usepackage[main, final]{neurips_2026}

\usepackage[utf8]{inputenc} % allow utf-8 input
\usepackage[T1]{fontenc}    % use 8-bit T1 fonts
\usepackage{hyperref}       % hyperlinks
\usepackage{url}            % simple URL typesetting
\usepackage{booktabs}       % professional-quality tables
\usepackage{multirow}       % multi-row cells in tables
\usepackage{colortbl}       % cell background colors in tables
\usepackage{amsfonts}       % blackboard math symbols
\usepackage{nicefrac}       % compact symbols for 1/2, etc.
\usepackage{microtype}      % microtypography
\usepackage{xcolor}         % colors
\definecolor{headerblue}{rgb}{0.816,0.910,1.000}
\definecolor{headerpink}{rgb}{0.984,0.902,0.902}
\usepackage{amsmath}
\usepackage{amssymb}
\usepackage{amsthm}
\usepackage{bm}
\usepackage{bbm}
\usepackage{graphicx}
\usepackage{mathtools}
\usepackage{enumitem}
\usepackage[most]{tcolorbox}
\usepackage{wrapfig}
\usepackage{subcaption}
\usepackage{enumitem}
\usepackage{algorithm}
\usepackage{algpseudocode}

\tcbuselibrary{skins,breakable}
\newtcolorbox{prompt}[1][Prompt]{
  colback=gray!5,
  colframe=gray!50!black,
  fonttitle=\bfseries,
  title=#1,
  sharp corners,
  boxrule=0.5pt,
  breakable,          %  <-- this is the key
  enforce breakable,  %  <-- belt-and-suspenders, forces it even in tricky contexts
}

\definecolor{todocolor}{HTML}{CC785C}
\newtheorem{theorem}{Theorem}

\newtheorem{proposition}{Proposition}
\newtheorem{corollary}{Corollary}

\usepackage[capitalise]{cleveref}
\crefname{appendix}{Appendix}{Appendices}
\Crefname{appendix}{Appendix}{Appendices}

\newcommand{\Dset}{\mathcal{D}}

\newcommand{\Vocab}{\mathcal{V}}
\newcommand{\E}{\mathbb{E}}
\newcommand{\N}{\mathbb{N}}

\newcommand{\R}{\mathbb{R}}

\newcommand{\argmax}{{\arg\max}}

\newcommand{\conf}{\mathrm{conf}}
\newcommand{\Dir}{\text{Dir}}

\newcommand{\alv}{\bm{\alpha}}

\newcommand{\ev}{\bm{e}}

\newcommand{\pv}{\bm{p}}
\newcommand{\qv}{\bm{q}}
\newcommand{\thetav}{\bm{\theta}}

\newcommand{\alk}{\alpha_k}
\newcommand{\alz}{\alpha_0}

\newcommand{\elli}{\ell^{(i)}}
\newcommand{\phk}{\hat{p}_k}
\newcommand{\pthetav}{p_{\thetav}}
\newcommand{\ui}{u^{(i)}}
\newcommand{\zi}{z^{(i)}}

\newcommand{\LMSE}{\mathcal{L}_{\text{MSE}}}
\newcommand{\LSFT}{\mathcal{L}_{\text{SFT}}}
\newcommand{\LSEDL}{\mathcal{L}_{\text{S-EDL}}}
\newcommand{\lmseweight}{\lambda_{\text{MSE}}}

\title{S-EDL: Eliciting Self-Evidence from Sequence Likelihoods for Semantic Calibration of LLMs}

\author{%
  Yawei Li\textsuperscript{1,3} \ 
  Jiazheng Li\textsuperscript{2} \ 
  David R\"ugamer\textsuperscript{1,3} \ 
  Bernd Bischl\textsuperscript{1,3} \ 
  Yulan He\textsuperscript{2} \ 
  Mina Rezaei\textsuperscript{1,3} \\[0.5em]
  \normalfont
  \textsuperscript{1}LMU Munich \quad
  \textsuperscript{2}King's College London \\
  \textsuperscript{3}Munich Center for Machine Learning (MCML)
}

\begin{document}

\maketitle

\input{sections/0_abstract.tex}

\input{sections/1_introduction.tex}

\input{sections/5_related_work.tex}

\input{sections/2_background.tex}

\input{sections/3_methodology.tex}

\input{sections/4_experiments.tex}

\input{sections/6_conclusion.tex}

\section*{Acknowledgment}
We gratefully acknowledge the funding that supported this work. YL and MR were supported by the Amazon Research Award. JL was supported by a PhD scholarship from AQA. YH was supported in part by the UK Engineering and Physical Sciences Research Council through a Turing AI Fellowship (grant no. EP/V020579/1, EP/V020579/2) and the Prosperity Partnership scheme (grant no. UKRI566).

\bibliographystyle{plainnat}
\bibliography{references}

\newpage
\input{sections/7_appendix.tex}

\newpage
\input{checklist.tex}

\end{document}

%% file: sections/0_abstract.tex
\begin{abstract}
Calibrating large language models (LLMs) in open-ended generation is uniquely challenging because many distinct token sequences can express the same underlying meaning. Prior calibration techniques target either fixed-set classification confidence or the binary correctness of individual strings, making them inherently ill-suited for the dynamic, multi-string nature of open-ended semantic outcomes. To bridge this gap, we introduce S-EDL, a novel evidential method that natively optimizes LLMs for semantic calibration. By eliciting evidence from the sequence likelihoods of the model itself, S-EDL constructs a differentiable Dirichlet prior over prompt-specific semantic classes. Optimizing the evidential loss under this prior yields a calibration-aware training signal over variable semantic classes, completely bypassing the need for fixed multiple-choice options. Evaluations across four open-ended benchmarks and three models demonstrate that S-EDL substantially reduces semantic calibration error while improving accuracy. By aligning generative likelihoods with semantic correctness, this work establishes a framework for deploying trustworthy LLMs in complex, open-ended applications. Code is available at: \url{https://github.com/sandylaker/s-edl}.
\end{abstract}

%% file: sections/1_introduction.tex
\section{Introduction}

Large language models (LLMs) are increasingly used for open-ended reasoning tasks, where reliable confidence estimates are important for deciding when model outputs can be trusted~\citep{liu2025uncertainty,mei2026reasoning,zollo2026unsupervised,li2026steering}. Early research has primarily focused on multiple-choice question answering (MCQA), where calibration can be defined over a small set of answer labels~\citep{wang2024blob,yang2024bayesian,rahmati2025clora,li2025calibrating}. Open-ended generation is more subtle: many distinct token sequences can express the same answer, and the probability assigned to a surface form does not necessarily reflect confidence in its semantic content. 

Prior work has studied open-ended uncertainty and calibration through several related targets. Semantic uncertainty metrics~\citep{kuhn2023semantic,farquhar2024detecting} sample multiple generations and group equivalent responses to evaluate the predictive distribution over semantic classes. Building on this perspective, semantic calibration~\citep{nakkiran2026trained} translates this grouping concept into an evaluation metric for open-ended generation. In parallel, verbal confidence calibration methods~\citep{kapoor2024large,stengel2024lacie,xu2024sayself, lu2026double, bani2026rewarding, damani2025beyond,zhang2026illusion} teach models to report a confidence score of correctness. However, this calibrates the binary correctness of a single response rather than the underlying distribution over alternative semantic answers. Thus, a gap remains: how to fine-tune open-ended LLMs under a semantic calibration objective. 

Evidential deep learning (EDL) natively bridges this gap. By placing a Dirichlet prior over categorical probabilities, it explicitly models the strength of evidence for multiple competing meanings, allowing calibration losses to act on distributions over semantic meanings. Building on this, we propose S-EDL, a self-evidenced EDL method tailored for semantic calibration. Unlike conventional EDL methods that require a fixed set of answer labels~\citep{sensoy2018evidential,chen2024r}, S-EDL accommodates the dynamic nature of open-ended generation by transforming the sequence likelihoods of the model into a differentiable evidence vector. This formulation allocates greater evidential support to semantic classes backed by highly probable strings. The evidence vector then parametrizes a Dirichlet prior over the semantic outcomes, providing a calibration-aware training signal. Our contributions are three-fold:

\begin{itemize}[leftmargin=1.5em, topsep=0pt, parsep=0pt, itemsep=0.5em]
\item We extend EDL from fixed-choice calibration to open-ended semantic calibration, allowing LLMs to optimize their confidence over diverse semantic meanings.

\item We introduce a novel parametrization that derives a differentiable evidence vector from the model's sequence likelihoods. This formulation also theoretically unifies post-hoc metrics such as Semantic Entropy~\citep{kuhn2023semantic,farquhar2024detecting} under a common evidential framework.

\item Across three models and four open-ended benchmarks, S-EDL jointly improves average accuracy and calibration over supervised fine-tuning (SFT), reducing average calibration error by up to roughly 1/3.
\end{itemize}

% \begin{figure}[t]
%     \centering
%     \includegraphics[width=\textwidth]{figures/s_edl_overview.pdf}
%     \caption{While standard EDL relies on a fixed set of labels (top), S-EDL adapts to open-ended generation by targeting \emph{semantic calibration} (bottom). To construct the evidence without predefined options, S-EDL samples responses from the LLM, clusters them by semantics, and aggregates their likelihoods. This maps dynamic generations into a Dirichlet prior over semantic distributions.}
%     \label{fig:placeholder}
% \end{figure}

%% file: sections/5_related_work.tex
\section{Related work}

\textbf{From MCQA to semantic calibration.}\quad Early LLM calibration focused on MCQA, where a fixed label set enables standard classification-style methods, including Bayesian low-rank adaptation~\citep{wang2024blob, yang2024bayesian, rahmati2025clora, shi2025trainingfree} and evidential approaches~\citep{li2025calibrating}. Open-ended generation introduces a level mismatch: many distinct sequences can express the same answer. Semantic uncertainty metrics~\citep{kuhn2023semantic, farquhar2024detecting, duan2024shifting, nikitin2024kernel,vipulanandan2026semantic, chen2024inside} sample generations and group by meaning, and semantic calibration~\citep{nakkiran2026trained} turns this view into a calibration target for open-ended generation. Our work follows this semantic line.

\textbf{Verbal and token-level calibration.}\quad A separate line of work asks models to emit a scalar correctness score~\citep{kapoor2024large,stengel2024lacie,xu2024sayself, lu2026double, bani2026rewarding, damani2025beyond,zhang2026illusion}, which calibrates a Bernoulli pushforward (correct vs.\ incorrect) of the output distribution rather than a categorical distribution over semantic outcomes. Meanwhile, SelfReflect~\citep{kirchhof2026selfreflect} captures the categorical distribution by a verbalized string. Token-level methods such as label smoothing~\citep{huang2025calibrated,szegedy2016rethinking}, temperature scaling~\citep{shen2024thermometer, lamb2026improving}, RCFT~\citep{xiao2025restoring}, and CATTO~\citep{parikh2026catto}, regularize per-token distributions, supervising a different granularity than the semantic marginal targeted here.

\textbf{Evidential deep learning.}\quad EDL~\citep{malinin2018predictive,sensoy2018evidential,deng2023uncertainty,chen2024r,joo2020being,charpentier2020posterior,charpentier2022natural,yoon2025uncertainty} places a Dirichlet prior over class probabilities through an evidence vector. Although previous work has applied EDL to calibrate LLMs in MCQA~\citep{li2025calibrating} and detect misbehavior in vision-language models~\citep{huang2026detecting}, S-EDL is, to the best of our knowledge, the first instantiation of EDL for calibrating LLMs in open-ended generation.

%% file: sections/2_background.tex
\section{Background}

\subsection{Semantic calibration in open-ended LLM generation}
\label{sec:semantic_calibration_background}

Our work focuses on calibration in open-ended generation, for which we adopt the framework of B-calibration~\citep{nakkiran2026trained}. Let $\Dset$ be a distribution over prompt-completion pairs $(x, y) \in \Vocab^* \times \Vocab^{L}$, where $x$ is a prompt over vocabulary $\Vocab^*$ and $y \in \Vocab^{L}$ is a ground-truth completion of length $L$. An autoregressive LLM $\pthetav$ with parameters $\bm{\theta}\in \Theta$ assigns probability $\pthetav(z|x)$ to each candidate response $z \in \Vocab^{L}$. Because distinct token sequences can convey the same meaning (e.g., ``half'' vs. ``0.5''), B-calibration~\citep{nakkiran2026trained} introduces a \emph{collapsing function} $B: \Vocab^* \times \Vocab^{L} \rightarrow \N$ mapping each prompt-response pair to a global \emph{semantic class}, a raw identifier $s \in \N$. We write $s := B(x, z)$ for a generation's class and $s^* := B(x, y)$ for the ground-truth class of $(x, y)$. The collapsing function $B$ can be implemented by a verification tool or an LLM-as-judge that clusters responses by meaning.
% \textbf{Induced semantic distribution and calibration.}\quad
% The LLM's sequence distribution pushes forward to a categorical distribution over semantic classes,
A \emph{semantic distribution} is then defined as the categorical distribution over semantic classes:
\begin{equation}
\label{eq:sem_dist}
    \pi_x(s) := \textstyle\sum_{z\,:\,B(x,z) = s} \pthetav(z|x), \quad s \in \N, z \in \Vocab^{L}.
\end{equation}
Further, denote the top class $s_x^{\text{top}} := \argmax_s \pi_x(s)$ and the \emph{semantic confidence} $\conf(x) := \pi_x(s_x^{\text{top}})$. \emph{Semantic calibration} requires that the confidence aligns with the true probability of correctness:
\begin{equation*}
    \E_{(x, y) \sim \Dset} \!\left[\, \mathbbm{1}\lbrace s_x^{\text{top}} = s^*\rbrace \,\big|\, \conf(x) \,\right] = \conf(x) \quad \text{almost surely.}
\end{equation*}
That is, among all prompts for which the model reports semantic confidence $c$, the top semantic class should indeed be correct with frequency $c$~\citep{nakkiran2026trained}.

\subsection{Evidential deep learning}
\label{sec:edl_background}

The calibration objective above asks the model to report not just which semantic class is most likely but also how confident it is in that belief. EDL~\citep{sensoy2018evidential,malinin2018predictive} provides a framework for the latter: rather than predicting a single categorical $\pv \in \Delta^K$ over $K$ classes, it predicts a distribution over $\pv$, accounting for uncertainty over class probabilities instead of training only against a point estimate.

\textbf{Dirichlet prior and evidence.}\quad
In this subsection, $k \in [K]$ denotes a generic class index for a standard fixed-$K$ classification problem, separate from the semantic-label notation of \cref{sec:semantic_calibration_background}. EDL places a Dirichlet prior $\Dir(\pv; \alv)$ on $\pv \in \Delta^K$, with Dirichlet parameters $\alv = (\alpha_k)_{k\in[K]} \in [1, +\infty)^K$ built from a non-negative \emph{evidence vector} $\ev = (e_k)_{k\in[K]} \in \R_{\geq 0}^K$ via $\alk = e_k + 1$.

Each $e_k$ represents the accumulated support for class $k$. Conventional EDL employs a classification model $f_{\thetav}$ to predict the evidence directly, $\ev = \text{SoftPlus}(f_{\thetav}(x))$, and the Dirichlet mean serves as the predictive distribution: $\phk = \E_{\pv \sim \Dir(\cdot; \alv)} [p_k] = \frac{\alk}{\alz}, \text{with \emph{Dirichlet strength} } \alz = \textstyle\sum_{k=1}^K \alk.$

\textbf{EDL training objective.}\quad
Given a one-hot target $\qv = (q_k)_{k\in[K]} \in \{0, 1\}^K$ with $q_k = \mathbbm{1}\lbrace k = t\rbrace$, \citet{sensoy2018evidential} learn the model's parameter $\bm{\theta}\in\Theta$ by minimizing the Bayes risk of MSE objective
\begin{equation} \label{eq:mse}
    \LMSE(\thetav) = \E_{\pv \sim \Dir(\cdot; \alv)} \|\qv - \pv\|_2^2 = \sum_{k=1}^K \left( q_k - \frac{\alk}{\alz} \right)^2 + \sum_{k=1}^K \frac{\alk (\alz - \alk)}{\alz^2 (\alz + 1)},
\end{equation}
which admits a closed form. Previous work IB-EDL~\citep{li2025calibrating} has shown that EDL methods are very effective in calibrating LLMs in the MCQA setting, but their approach is limited to a fixed set of options such as A/B/C/D. In the following section, we move a step forward and instantiate this framework in the open-ended generation setting, deriving $\ev$ from the LLM's own generations.

%% file: sections/3_methodology.tex
\section{Methodology}
\label{sec:method}

\subsection{Motivation}
\label{sec:motivation}

Given EDL's effectiveness for calibrating LLMs in MCQA, we now discuss the adaptations required to apply it in open-ended generation.
Instantiating the EDL framework in this setting hinges on one design choice: how to parametrize the evidence vector $\ev$. Conventional EDL learns $\ev$ through a dedicated prediction head whose output dimension equals the number of classes. This recipe does not transfer to our setting for two reasons. First, the relevant semantic class inventory is prompt-dependent and not known in advance, so there is no canonical fixed output space. Second, and most fundamentally, the distribution whose calibration we ultimately care about is $\pthetav(\cdot | x)$; routing evidence through an independent head decouples it from the very object we aim to calibrate. We therefore look for a parametrization that derives $\ev$ from $\pthetav(\cdot | x)$.

A natural choice for constructing $\ev$ is utilizing $\pi_x$ itself, which already encodes how $\pthetav(\cdot|x)$ distributes mass over semantic classes. Since $\mathcal{V}^L$ is typically vast, \citet{nakkiran2026trained} estimate $\pi_x$ from $N$ i.i.d. Monte-Carlo samples $z^{(i)} \sim \pthetav(\cdot|x)$ with semantic classes $s^{(i)} = B(x, z^{(i)})$, which induce the prompt-specific semantic class set $\mathcal{S}_x := \{s^{(i)} : i \in [N]\} \subset \mathbb{N}$. Further details regarding the empirical construction of $\mathcal{S}_x$ in edge cases are in \cref{sec:unobserved_gt}. The plug-in estimator of $\pi_x$ is then
\begin{equation}\label{eq:histogram}
    \hat\pi_x(s) = \tfrac{1}{N} \textstyle\sum_{i=1}^{N} \mathbbm{1}[s^{(i)} = s], \quad s \in \mathcal{S}_x.
\end{equation}
However, this histogram is not a suitable evidence vector for EDL training. Once the samples are drawn, the class indicators are constants with respect to $\thetav$, so an evidence vector proportional to $\hat{\pi}_x$ provides no differentiable path from the EDL loss back to the model parameters. In addition, the histogram treats all samples equally, regardless of their likelihood under $\pthetav(\cdot \mid x)$. This is misaligned with the evidential interpretation of $\ev$: a semantic class supported by high-likelihood generations should contribute more evidence than one represented only by low-likelihood samples.

\textbf{Desiderata.}\quad These observations lead to two desiderata for parametrizing evidence $\ev$:
\begin{enumerate}[label=(\alph*),leftmargin=1.8em, nosep, before=\vspace{-\parskip}, after=\vspace{-\parskip}]
    \item The parametrization should assign more evidence to semantic classes supported by high-likelihood completions, so that likely completions contribute more to evidence.
    \item The parametrization should remain differentiable through the sample log-likelihoods, so that the EDL loss can update the likelihoods underlying the sampled semantic estimate.
\end{enumerate}

Next, we instantiate these desiderata through a Monte-Carlo estimate of a tempered semantic distribution, and use this estimate to define a Dirichlet prior over the sampled semantic classes.

\subsection{Formulation of self-evidence}
\label{sec:formulation}

Throughout this subsection, we fix the input prompt $x$ and drop the subscript from prompt-specific quantities (e.g., $\pi_x$ to $\pi$, $\mathcal{S}_x$ to $\mathcal{S}$) to lighten notation; all such objects remain implicitly conditional on $x$. We work directly with semantic labels $s \in \mathcal{S}$, without introducing a separate local class index. In addition, we use the term \emph{self-evidence} to refer to evidential support constructed from the LLM's own Monte-Carlo samples. At the sample level, each generation $\zi$ receives a scalar \emph{self-evidence weight} $\ui$; at the class level, these weights aggregate into the \emph{self-evidence vector} $\ev$.

\textbf{Tempered semantic distribution.}\quad
For a scalar $\beta \geq 1$, we define the \emph{tempered semantic distribution} over classes $s \in \mathcal{S}$ and $z,z' \in \Vocab^{L}$ as
\begin{equation}\label{eq:tempered}
    \pi_\beta(s) := \frac{\sum_{z\,:\,B(x,z)=s} \pthetav(z | x)^\beta}{\sum_{z' \in \Vocab^{L}} \pthetav(z' | x)^\beta}.
\end{equation}
For $\beta = 1$, \cref{eq:tempered} recovers the standard semantic distribution $\pi$ in \cref{eq:sem_dist}.
For $\beta > 1$, raising per-sequence likelihoods to a higher power shifts mass toward classes that the model supports with high-confidence sequences, formalizing desideratum~(a). 
% The parameter $\beta$ is a single scalar controlling how strongly confident generations dominate the aggregate.

\textbf{Self-evidence estimator.}\quad
Direct evaluation of \cref{eq:tempered} is intractable because both the classwise numerator and the global normalizer require sums over $\Vocab^{L}$. We therefore estimate $\pi_\beta$ from the same $N$ Monte-Carlo samples introduced in \cref{sec:motivation}. Writing $\elli := \log \pthetav(\zi | x)$ for teacher-forcing log-likelihood for sample $i$, we assign to each sample the self-evidence weight
\begin{equation} \label{eq:self_evidence_weight}
    \ui := \exp \!\left( (\beta - 1)\, \elli \right) = \pthetav(\zi | x)^{\beta - 1}.
\end{equation}
The exponent is $\beta - 1$ rather than $\beta$ because the samples $\zi$ are themselves drawn from $\pthetav(\cdot | x)$, so multiplying by $\pthetav(\zi | x)^{\beta - 1}$ recovers the $\pthetav^\beta$ terms of $\pi_\beta$ in \cref{eq:tempered} (made precise in \cref{thm:consistency}). Because $\ui$ is a smooth function of $\elli$, gradients can pass through the sample log-likelihoods to $\thetav$ and satisfy desideratum~(b).
We aggregate the weights within each semantic class and normalize:
\begin{equation}\label{eq:estimator}
    U_s := \!\!\!\!\!\!\!\! \sum_{i\in[N]\,:\,s^{(i)} = s} \!\!\!\!\!\!\!\!\ui, \qquad \hat\pi_\beta(s) := \frac{U_s}{\sum_{r \in \mathcal{S}} U_r}.
\end{equation}
A natural question is whether this estimator $\hat{\pi}_\beta$ recovers the true tempered distribution $\pi_\beta$ as the number of samples grows. The next result confirms this (see the proof in \cref{sec:app:proof_consistency}).
\begin{proposition}\label{thm:consistency}
As $N \to \infty$, $\hat\pi_\beta(s) \xrightarrow{\text{a.s.}} \pi_\beta(s)$ for every $s \in \mathcal{S}$.
\end{proposition}

\textbf{Connection to Semantic Entropy and the practical choice of $\beta$.}\quad
Semantic Entropy (SE)~\citep{kuhn2023semantic,farquhar2024detecting} also aggregates samples by their sequence probabilities $\pthetav(\zi | x)$, which under \cref{eq:self_evidence_weight} corresponds to the special case $\beta = 2$. In our experiments, we adopt a small value $\beta = 1.2$ (ablation in \cref{sec:ablation}). The choice of $\beta$ reflects the two methods' different purposes: SE measures the uncertainty as a post-hoc metric, while S-EDL uses $\pi_\beta$ as a training signal that drives gradient updates (\cref{sec:gradient}). 

\textbf{Self-evidence vector and Dirichlet prior.}\quad
We finally convert $\hat\pi_\beta$ into an evidential prior by distributing the sample budget $N$ over the semantic classes according to the estimated tempered distribution:
\begin{equation}\label{eq:evidence_prior}
    e_s := N \cdot \hat\pi_\beta(s), \qquad \alpha_s := e_s + 1, \qquad s \in \mathcal{S}.
\end{equation}
Because $\hat\pi_\beta$ is a probability vector, the Dirichlet strength is determined a priori by the sample budget and the semantic class count:
%\begin{equation} \label{eq:strength}
    $\alz = \sum_{s \in \mathcal{S}} \alpha_s = N + |\mathcal{S}|$.
%\end{equation}
This differs from conventional EDL, where the model freely predicts the per-class evidences $e_k$ (equivalently, the Dirichlet parameters $\alpha_k = e_k + 1$), so the total strength $\alz = \sum_k \alpha_k$ is not fixed a priori. Here $\alz = N + |\mathcal{S}|$ is pinned by construction, and all flexibility in the prior is routed through $\hat\pi_\beta$. The resulting $\Dir(\cdot; \alv)$ is fed into the MSE loss of \cref{eq:mse}, completing the differentiable path promised by desideratum~(b), as the entire computational chain $\elli \rightarrow \ui \rightarrow \hat{\pi}_\beta \rightarrow \ev \rightarrow \alv \rightarrow \LMSE$ is differentiable w.r.t. $\thetav$.

In training, we optimize the $\LMSE$ jointly with standard supervised fine-tuning (SFT). The SFT loss retains the next-token learning signal, while the evidential term provides the semantic-level calibration signal. Let $\LSFT$ denote the next-token cross-entropy loss on a pair $(x, y)$. Our final objective then is
\begin{equation*}
    \LSEDL(\thetav) = \LSFT(\thetav) + \lmseweight \LMSE(\thetav),
\end{equation*}
where $\lmseweight$ is the weight for the evidential loss. The next subsection analyzes the gradient of $\LMSE$.

\subsection{Gradient structure of S-EDL}
\label{sec:gradient}

Having constructed the self-evidence vector and the Dirichlet prior, we next examine the optimization signal S-EDL produces. The result below shows that the gradient with respect to per-sample log-likelihood factors into a normalized sample weight and a class-level correction term.

% \begin{theorem}\label{thm:ll_gradient}
% For each class $s \in \mathcal{S}$, let $d_s := \partial \LMSE / \partial \hat\pi_\beta(s)$ denote the gradient of the MSE loss with respect to the estimated tempered distribution at class $s$, and let $\bar d := \sum_{r \in \mathcal{S}} \hat\pi_\beta(r)\, d_r$ denote the $\hat\pi_\beta$-weighted mean of these per-class gradients. For each sample $i \in [N]$, let $w^{(i)} := \ui / \sum_{j=1}^{N} u^{(j)}$ denote the normalized self-evidence weight of sample $i$, i.e., its share of the total self-evidence across all $N$ samples. Then
% \begin{equation}\label{eq:ds}
%     d_s = \frac{N \left[\, 2\alpha_s - 2(N + |\mathcal{S}| + 1)\, q_s + 1 \,\right]}{(N + |\mathcal{S}|)(N + |\mathcal{S}| + 1)},
% \end{equation}
% where $q_s := \mathbbm{1}[s = s^*]$, and
% \begin{equation}\label{eq:dl_dell}
%     \frac{\partial \LMSE}{\partial \elli} = (\beta - 1)\, w^{(i)} \left( d_{s^{(i)}} - \bar d \right).
% \end{equation}
% \end{theorem}

\begin{theorem}\label{thm:ll_gradient}
For each class $s \in \mathcal{S}$, let $d_s := \partial \LMSE / \partial \hat\pi_\beta(s)$, and let $\bar d := \sum_{r \in \mathcal{S}} \hat\pi_\beta(r)\, d_r$ denote their weighted mean. Let $w^{(i)} := \ui / \sum_{j=1}^{N} u^{(j)}, i\in[N]$ be the normalized sample weight. Then
\begin{equation}\label{eq:dl_dell}
    \frac{\partial \LMSE}{\partial \elli} = (\beta - 1)\, w^{(i)} \left( d_{s^{(i)}} - \bar d \right).
\end{equation}
\end{theorem}

The proof is shown in \cref{sec:app:proof_ll_gradient}. \cref{thm:ll_gradient} shows that each sample's contribution is determined by the sample weight $w^{(i)}$ and the class-level signal $d_{s^{(i)}} - \bar d$.
Since $w^{(i)} \propto \pthetav(\zi | x)^{\beta - 1}$, samples with larger model likelihoods receive larger gradient weights. Optimization is therefore concentrated on the ``high-probability region'' of $\pthetav(\cdot | x)$. We next characterize its direction on the two semantic classes most relevant to calibration: the ground-truth class and the most overconfident incorrect class.

\begin{corollary}\label{thm:sign_of_grad}
Assume $\beta>1$. For any sample $i\in[N]$ in the correct class ($s^{(i)}=s^*$), $\partial \LMSE / \partial \ell^{(i)} \leq 0$, with strict inequality if $\hat\pi_\beta(s^*) < 1$. Conversely, for any sample in the most overconfident incorrect class $s^- := \argmax_{s \neq s^*} \hat\pi_\beta(s)$, it holds that $\partial \LMSE/\partial \ell^{(i)} \ge 0$, with strict inequality if $\hat\pi_\beta(s^*) > 0$. 
\end{corollary}
The proof is detailed in \cref{sec:app:proof_sign_of_grad}. Consequently, minimizing $\LMSE$ steers the model in the desired direction. In any realistic training scenario where the model's semantic distribution has not entirely collapsed ($0<\hat\pi_\beta(s^*)<1$), the objective actively promotes the sequences in ground-truth semantic class, while simultaneously penalizing those in the most overconfident incorrect one.
The next corollary shows that the per-sample gradients are balanced across the samples.

\begin{corollary}\label{thm:conservation}
For the gradients $\partial \LMSE/\partial \elli$ in \cref{thm:ll_gradient}, their sum $\sum_{i=1}^{N} \partial \LMSE/\partial \elli = 0$. 
\end{corollary}
See \cref{sec:app:proof_conservation} for the proof. S-EDL therefore \emph{conservatively reallocates} log-likelihoods rather than uniformly inflating or suppressing them: combined with the sign structure above, the update shifts probability mass away from overconfident incorrect classes and toward the ground-truth class.

%% file: sections/4_experiments.tex
\section{Experiments}
\label{sec:experiments}

\subsection{Experimental setup}
\label{sec:exp_setup}

\textbf{Datasets.}\quad
We conduct training and evaluation on four open-ended STEM datasets: GSM8K~\citep{cobbe2021training}, MATH~\citep{hendrycks2021measuring}, ChemistryQA~\citep{wei2021chemistryqa}, and FormulaReasoning~\citep{li2024formulareasoning}, a physics QA dataset. These benchmarks share two properties that make them well suited to our study. First, final answers admit canonical numerical or symbolic forms, so the collapsing function that groups sampled responses into semantic classes (a symbolic verifier or an LLM-as-judge) incurs comparatively low evaluation noise. Second, they often require multi-step reasoning, so calibration over complete semantic outcomes is more relevant than calibration over individual tokens. We train on the training split of each dataset and evaluate on the corresponding test split, with the exception of MATH, where we follow standard practice~\citep{hendrycks2021measuring,lightman2024let} and evaluate on MATH-500, a subset of the MATH test split widely used in other literature. Dataset construction is detailed in \cref{sec:app:dataset_construction}.

\textbf{Models and training.}\quad
We fine-tune Qwen3-8B-Base~\citep{qwen3technicalreport}, Ministral3-8B-Base~\citep{liu2026ministral}, and MiMo-7B-Base~\citep{xiaomi2025mimo} with LoRA adapters~\citep{hu2022lora} of rank $r_{\text{LoRA}} = 16$ and $\alpha_{\text{LoRA}} = 32$. For S-EDL, $\beta = 1.2$ and MSE loss weight $\lmseweight = 5.0$ are held fixed across all models and datasets. In addition, the number of generations per prompt during training $N_{\text{train}}$ is set as large as GPU memory permits. On GSM8K, where outputs are comparatively short, we use $N_{\text{train}} = 32$ for all models. On the remaining three datasets, the memory ceiling reduces to $N_{\text{train}} = 8$ for Qwen3-8B-Base and $N_{\text{train}} = 12$ for Ministral3-8B-Base and MiMo-7B-Base. Complete optimization hyperparameters are in \cref{sec:app:training_hyperparameters}.

\textbf{Evaluation protocol.}\quad
We use $\texttt{top\_p} = 0.9$ for all models, with \texttt{temperature} $0.8$ for Qwen3-8B-Base, $0.6$ for Ministral3-8B-Base, and $0.7$ for MiMo-7B-Base. The per-model temperature is chosen on a held-out validation split to approximately match the entropy of the respective model, and is fixed across fine-tuning methods. 
Following \citet{nakkiran2026trained}, we report two semantic-level metrics computed from $N_{\text{eval}} = 50$ generations per question: \emph{Semantic Accuracy} (S-Acc.), the accuracy of majority voting over semantically clustered responses, and \emph{Smooth ECE} (smECE), a kernel-smoothed calibration error metric. As a single-sequence-level metric, we additionally report \emph{Pass@1} accuracy.

\textbf{Baselines.}\quad Besides standard SFT, we compare S-EDL against baselines spanning three categories: (i) \emph{token-level calibration}, including Label Smoothing (LS)~\citep{huang2025calibrated} and RCFT~\citep{xiao2025restoring}, which regularize token-wise output distributions; (ii) \emph{Bayesian ensembling}, represented by BLoB~\citep{wang2024blob}; and (iii) \emph{sequence-level correctness optimization}, represented by Dr.GRPO~\citep{liu2025understanding} with a binary $0/1$ correctness reward. We include Dr.GRPO as a control that isolates the effect of sequence-level supervision on calibration: under the binary reward, it maximizes the probability of the correct semantic class without an explicit calibration term, which lets us assess whether operating at the sequence level alone accounts for the gains of S-EDL or whether the Brier-style calibration objective like ~\cref{eq:mse} is required. The details about the baselines are outlined in \cref{sec:app:baseline_hyperparameters}.

\subsection{Calibration performance on in-distribution datasets}
\label{sec:in_distribution_evaluation}

\begin{table}[t]
\centering
\caption{Calibration and accuracy across four open-ended question answering datasets. We report the single-sequence metric \tcbox[on line, colback=headerblue, colframe=headerblue, boxrule=0pt, arc=4pt, left=0.5pt, right=0.5pt, top=0pt, bottom=0pt]{Pass@1} ($\uparrow$) and two semantic-level metrics: \tcbox[on line, colback=headerpink, colframe=headerpink, boxrule=0pt, arc=4pt, left=0.5pt, right=0.5pt, top=0pt, bottom=0pt]{Semantic Accuracy} (S-Acc., $\uparrow$) and \tcbox[on line, colback=headerpink, colframe=headerpink, boxrule=0pt, arc=4pt, left=0.5pt, right=0.5pt, top=0pt, bottom=0pt]{Smooth ECE} (smECE, $\downarrow$). Entries show means over three seeds with standard deviations as smaller subscripts ($\text{mean}_{\text{std}}$).
The Average columns report means across the four datasets.
The \textbf{best} and \underline{second-best} results are highlighted.}
\label{tab:in_distribution_acc_vs_smece}
\resizebox{\textwidth}{!}{%
\setlength{\tabcolsep}{1.5pt}
\begin{tabular}{l|ccc|ccc|ccc|ccc|ccc}
\toprule
\multirow{2}{*}{\textbf{Method}} & \multicolumn{3}{c|}{\textbf{GSM8K}} & \multicolumn{3}{c|}{\textbf{MATH-500}} & \multicolumn{3}{c|}{\textbf{ChemistryQA}} & \multicolumn{3}{c|}{\textbf{FormulaReasoning}} & \multicolumn{3}{c}{\textbf{Average}} \\
& \cellcolor{headerblue}Pass@1$\uparrow$ & \cellcolor{headerpink}S-Acc.$\uparrow$ & \cellcolor{headerpink}smECE$\downarrow$ & \cellcolor{headerblue}Pass@1$\uparrow$ & \cellcolor{headerpink}S-Acc.$\uparrow$ & \cellcolor{headerpink}smECE$\downarrow$ & \cellcolor{headerblue}Pass@1$\uparrow$ & \cellcolor{headerpink}S-Acc.$\uparrow$ & \cellcolor{headerpink}smECE$\downarrow$ & \cellcolor{headerblue}Pass@1$\uparrow$ & \cellcolor{headerpink}S-Acc.$\uparrow$ & \cellcolor{headerpink}smECE$\downarrow$ & \cellcolor{headerblue}Pass@1$\uparrow$ & \cellcolor{headerpink}S-Acc.$\uparrow$ & \cellcolor{headerpink}smECE$\downarrow$ \\
\midrule
\multicolumn{16}{c}{\textit{Fine-tuned model: Qwen3-8B-Base}} \\
\midrule
SFT & 70.8$_{1.2}$ & \underline{93.4}$_{0.4}$ & 20.2$_{0.4}$ & 49.4$_{2.8}$ & 68.7$_{2.7}$ & 16.1$_{0.6}$ & 47.9$_{0.7}$ & 52.9$_{1.4}$ & 20.7$_{0.6}$ & 75.6$_{0.8}$ & 87.5$_{0.3}$ & 8.9$_{0.3}$ & 60.9 & 75.6 & 16.5 \\
RCFT & 79.1$_{0.2}$ & 92.8$_{0.1}$ & 12.9$_{0.2}$ & 50.1$_{0.3}$ & 71.1$_{1.4}$ & 17.5$_{1.6}$ & 49.7$_{3.9}$ & 55.4$_{4.3}$ & 15.1$_{1.5}$ & 78.4$_{0.3}$ & \underline{89.5}$_{0.2}$ & 8.9$_{0.1}$ & 64.3 & 77.6 & 13.6 \\
LS & 79.1$_{1.0}$ & 92.7$_{0.2}$ & 12.5$_{0.6}$ & 49.6$_{0.2}$ & 70.7$_{0.7}$ & 17.5$_{0.7}$ & 50.7$_{0.7}$ & 58.1$_{0.3}$ & \textbf{11.8}$_{0.7}$ & 76.8$_{1.0}$ & 88.9$_{0.4}$ & 9.4$_{0.8}$ & 64.0 & 77.6 & 12.8 \\
BLoB & 79.1$_{0.1}$ & 92.5$_{0.9}$ & 12.4$_{0.3}$ & 49.9$_{0.3}$ & 69.5$_{1.5}$ & 15.9$_{1.6}$ & 52.4$_{0.9}$ & 57.9$_{0.2}$ & 14.9$_{1.3}$ & 78.4$_{1.3}$ & \textbf{90.3}$_{0.6}$ & 10.6$_{0.1}$ & 64.9 & 77.6 & 13.4 \\
Dr.GRPO & \underline{80.8}$_{0.3}$ & 93.2$_{0.6}$ & \underline{11.8}$_{0.6}$ & \underline{56.2}$_{0.5}$ & \underline{74.6}$_{0.6}$ & \underline{15.5}$_{0.5}$ & \textbf{53.0}$_{0.3}$ & \underline{59.2}$_{0.2}$ & 14.4$_{1.4}$ & \underline{79.5}$_{0.4}$ & \underline{89.5}$_{0.5}$ & \underline{7.2}$_{0.6}$ & \underline{67.4} & \underline{79.1} & \underline{12.2} \\
S-EDL & \textbf{84.1}$_{1.0}$ & \textbf{94.3}$_{0.3}$ & \textbf{9.5}$_{0.2}$ & \textbf{64.6}$_{0.5}$ & \textbf{81.3}$_{1.2}$ & \textbf{14.1}$_{0.7}$ & \underline{52.8}$_{0.2}$ & \textbf{59.9}$_{0.6}$ & \underline{12.6}$_{0.9}$ & \textbf{80.6}$_{0.2}$ & \underline{89.5}$_{0.3}$ & \textbf{5.8}$_{0.7}$ & \textbf{70.5} & \textbf{81.3} & \textbf{10.5} \\
\midrule
\multicolumn{16}{c}{\textit{Fine-tuned model: Ministral3-8B-Base}} \\
\midrule
SFT & 64.4$_{0.4}$ & 80.9$_{0.4}$ & 12.0$_{0.3}$ & 35.6$_{0.7}$ & 54.2$_{1.3}$ & 11.2$_{0.2}$ & \textbf{48.2}$_{0.8}$ & \underline{55.4}$_{0.6}$ & 15.0$_{1.0}$ & 72.9$_{2.1}$ & 85.3$_{0.8}$ & 10.6$_{0.5}$ & 55.3 & 69.0 & 12.2 \\
RCFT & 63.4$_{0.6}$ & 78.8$_{0.5}$ & \underline{10.9}$_{0.2}$ & 37.6$_{0.4}$ & \underline{55.7}$_{1.6}$ & 10.8$_{1.6}$ & \underline{48.1}$_{1.0}$ & 55.1$_{0.3}$ & 13.1$_{1.4}$ & 74.7$_{1.8}$ & 86.5$_{0.8}$ & 9.0$_{1.0}$ & 55.9 & 69.0 & \underline{10.9} \\
LS & 58.5$_{0.3}$ & 77.2$_{1.0}$ & 13.4$_{0.1}$ & 34.1$_{4.5}$ & 52.7$_{4.7}$ & 11.4$_{1.6}$ & 47.3$_{1.6}$ & 53.9$_{0.3}$ & \underline{12.8}$_{1.3}$ & 69.0$_{1.7}$ & 86.2$_{3.4}$ & 13.0$_{2.9}$ & 52.2 & 67.5 & 12.6 \\
BLoB & 63.5$_{0.2}$ & 80.5$_{0.3}$ & 13.1$_{0.5}$ & 37.4$_{1.0}$ & 54.3$_{0.1}$ & \underline{9.4}$_{1.0}$ & 47.7$_{0.8}$ & 53.8$_{1.6}$ & 14.9$_{1.4}$ & 74.4$_{1.0}$ & 86.1$_{0.4}$ & 8.8$_{1.1}$ & 55.7 & 68.7 & 11.6 \\
Dr.GRPO & \underline{76.5}$_{1.7}$ & \underline{83.0}$_{0.8}$ & 15.5$_{3.0}$ & \underline{48.0}$_{1.1}$ & 55.4$_{0.4}$ & \textbf{7.7}$_{0.5}$ & 45.0$_{0.8}$ & 52.8$_{1.0}$ & 14.2$_{0.4}$ & \underline{80.4}$_{0.9}$ & \underline{89.5}$_{0.6}$ & \underline{8.2}$_{0.6}$ & \underline{62.5} & \underline{70.2} & 11.4 \\
S-EDL & \textbf{82.0}$_{0.5}$ & \textbf{89.4}$_{0.5}$ & \textbf{5.4}$_{0.4}$ & \textbf{50.8}$_{1.3}$ & \textbf{68.2}$_{0.7}$ & 11.3$_{1.3}$ & 47.8$_{1.0}$ & \textbf{55.9}$_{1.9}$ & \textbf{11.5}$_{1.2}$ & \textbf{81.0}$_{0.9}$ & \textbf{89.7}$_{0.3}$ & \textbf{5.9}$_{0.8}$ & \textbf{65.4} & \textbf{75.8} & \textbf{8.6} \\
\midrule
\multicolumn{16}{c}{\textit{Fine-tuned model: MiMo-7B-Base}} \\
\midrule
SFT & 73.8$_{1.1}$ & 87.5$_{0.1}$ & 11.5$_{0.2}$ & 34.0$_{0.5}$ & 53.9$_{0.1}$ & 14.3$_{0.4}$ & \textbf{46.7}$_{0.9}$ & 53.4$_{0.7}$ & 11.9$_{0.3}$ & 70.4$_{0.7}$ & 83.9$_{0.6}$ & 10.9$_{1.0}$ & 56.2 & 69.7 & 12.1 \\
RCFT & 70.9$_{0.8}$ & 86.8$_{0.4}$ & 14.5$_{0.6}$ & 33.7$_{0.3}$ & 53.3$_{0.6}$ & 13.0$_{0.7}$ & 43.6$_{1.0}$ & 54.0$_{0.4}$ & 7.9$_{1.9}$ & 63.9$_{1.2}$ & 82.1$_{1.7}$ & 15.1$_{2.0}$ & 53.0 & 69.1 & 12.7 \\
LS & 71.3$_{0.9}$ & 87.3$_{0.6}$ & 13.5$_{0.2}$ & 32.9$_{0.3}$ & 52.5$_{0.8}$ & 13.2$_{0.8}$ & 41.8$_{1.5}$ & 51.7$_{2.4}$ & \textbf{6.8}$_{0.9}$ & 67.9$_{3.4}$ & \underline{85.5}$_{4.5}$ & 14.2$_{1.2}$ & 53.5 & 69.3 & 11.9 \\
BLoB & 71.2$_{0.2}$ & 88.0$_{0.3}$ & 14.0$_{0.1}$ & 33.4$_{0.4}$ & 50.9$_{2.8}$ & \underline{11.6}$_{2.4}$ & 43.9$_{2.0}$ & 53.9$_{1.5}$ & 9.2$_{1.0}$ & 64.4$_{2.1}$ & 81.1$_{2.9}$ & 13.7$_{2.0}$ & 53.2 & 68.5 & 12.1 \\
Dr.GRPO & \textbf{81.0}$_{0.9}$ & \underline{90.4}$_{1.1}$ & \textbf{7.7}$_{1.0}$ & \underline{51.0}$_{1.3}$ & \underline{66.3}$_{0.3}$ & \textbf{10.7}$_{1.0}$ & \underline{45.3}$_{1.0}$ & \underline{54.9}$_{1.1}$ & 11.1$_{1.1}$ & \underline{72.7}$_{1.6}$ & 84.8$_{0.4}$ & \underline{9.7}$_{0.4}$ & \underline{62.5} & \underline{74.1} & \underline{9.8} \\
S-EDL & \underline{81.0}$_{0.4}$ & \textbf{90.7}$_{0.2}$ & \underline{8.5}$_{0.4}$ & \textbf{56.0}$_{0.4}$ & \textbf{72.6}$_{1.7}$ & 12.3$_{2.5}$ & 45.3$_{1.7}$ & \textbf{56.1}$_{1.4}$ & \underline{7.5}$_{0.6}$ & \textbf{74.9}$_{0.4}$ & \textbf{87.3}$_{0.4}$ & \textbf{9.4}$_{1.3}$ & \textbf{64.3} & \textbf{76.7} & \textbf{9.4} \\
\bottomrule
\end{tabular}%
}
\end{table}

% \paragraph{Sequence-level regularization is more effective than token-level for semantic calibration.}
As shown in \cref{tab:in_distribution_acc_vs_smece}, S-EDL attains the lowest average smECE across all three models, improving over the strongest token-level baseline by 2.3 points on Qwen3-8B-Base (10.5 vs.\ LS's 12.8), 2.3 points on Ministral3-8B-Base (8.6 vs.\ RCFT's 10.9), and 2.5 points on MiMo-7B-Base (9.4 vs.\ LS's 11.9). 
Dr.GRPO also outperforms every token-level baseline in average smECE on Qwen3-8B-Base and MiMo-7B-Base, but its 11.4 on Ministral3-8B-Base does not improve over RCFT's 10.9, suggesting that operating at the sequence level alone is not sufficient and that the choice of training signal also matters.

The underperformance of token-level methods may reflect a level mismatch between supervision and evaluation: LS and RCFT calibrate per-token conditional distributions, whereas sequence-level objectives act on complete responses; S-EDL further applies its loss to a sampled estimate of the semantic distribution, matching the granularity of the semantic metrics.
This mismatch may be amplified in chain-of-thought generation, where a small subset of decisive tokens (e.g.,  a branching choice indicated by ``alternatively'') largely determines the semantic outcome~\citep{wang2025beyond}, while most tokens are routine continuations; per-token regularization applies to both groups, potentially diluting the signal on the tokens that actually matter.

% \textbf{S-EDL alleviates the accuracy-calibration trade-off that other baselines incur.}\quad
Furthermore, we observe that several baselines improve over SFT on one metric at the cost of the other.
LS provides a representative example: on MiMo-7B-Base, it lowers average Pass@1 from $56.2$ to $53.5$ while improving smECE only marginally, from $12.1$ to $11.9$.
Averaged across the four datasets, S-EDL improves over SFT on all three metrics for every model, e.g., with $+9.6$ Pass@1, $+5.7$ S-Acc, and $-6.0$ smECE on Qwen3-8B-Base. \cref{fig:accuracy_calibration} in \cref{sec:app:acc_calibration_plane} visualizes the averaged S-Acc $\times$ smECE plane. S-EDL consistently lies on the Pareto frontier of the plane.
The structure of the evidential loss may help explain this behavior. \cref{eq:mse} is a Brier-style proper scoring rule, which penalizes both classification error and miscalibration simultaneously. Dr.GRPO with a binary $0/1$ reward, maximizes the probability of the correct semantic class without an explicit calibration term, and may improve accuracy at the cost of miscalibration.

\subsection{Cross-domain evaluation}
\label{sec:cross_domain_evaluation}

To test the transferability of each method, we train on MATH and evaluate on three other cross-domain datasets as shown in Table~\ref{tab:cross_domain}. This setting tests not only whether accuracy transfers, but also whether semantic confidence remains aligned with correctness.
The grouping into sequence-level (Dr.GRPO, S-EDL) and token-level/Bayesian (LS, RCFT, BLoB) baselines from \cref{sec:in_distribution_evaluation} separates more sharply in average Pass@1 in cross-domain evaluation than in-domain.
The average Pass@1 gap between the better-performing sequence-level method and the strongest token-level/Bayesian baseline widens from $5.6 / 9.5 / 10.8$ points in-domain to $13.8 / 16.5 / 13.6$ points cross-domain on Qwen3-8B-Base, Ministral3-8B-Base, and MiMo-7B-Base, respectively.

A consistent diagnostic appears against the SFT reference: on Qwen3-8B-Base and MiMo-7B-Base, none of LS, RCFT, or BLoB improves average cross-domain smECE over plain SFT.
One possible interpretation is that per-token regularization and Bayesian ensembling shape conditional distributions tied to the lexical statistics of the training distribution, so the resulting effect on the implied semantic distribution need not be preserved when input statistics change. 
In contrast, objectives defined over complete responses are less tied to token-level statistics and better aligned with the semantic nature.

The two sequence-level methods themselves, however, do not behave identically under shift.
In cross-domain evaluation, the average Pass@1 of S-EDL exceeds that of Dr.GRPO by $7.3 / 6.3 / 3.3$ points across the three models.
Its average cross-domain smECE is lower by $4.7 / 3.0 / 0.1$ points.
While Dr.GRPO is effective at optimizing correctness, its binary reward does not explicitly regularize the semantic confidence.
The performance gap between the two methods suggests that coupling sequence-level training with an explicit evidential calibration objective could contribute to more generalizable improvements in accuracy and calibration.

\begin{table}[t]
\centering
\caption{Cross-domain evaluation. A $\rightarrow$ B denotes training on dataset A and evaluating on dataset B.}
\label{tab:cross_domain}
\resizebox{0.96\textwidth}{!}{%
\setlength{\tabcolsep}{1.5pt}
\begin{tabular}{l|ccc|ccc|ccc|ccc}
\toprule
\multirow{2}{*}{\textbf{Method}} & \multicolumn{3}{c|}{\textbf{MATH $\rightarrow$ GSM8K}} & \multicolumn{3}{c|}{\textbf{MATH $\rightarrow$ ChemistryQA}} & \multicolumn{3}{c|}{\textbf{MATH $\rightarrow$ FormulaReasoning}} & \multicolumn{3}{c}{\textbf{Average}} \\
& \cellcolor{headerblue}Pass@1$\uparrow$ & \cellcolor{headerpink}S-Acc.$\uparrow$ & \cellcolor{headerpink}smECE$\downarrow$ & \cellcolor{headerblue}Pass@1$\uparrow$ & \cellcolor{headerpink}S-Acc.$\uparrow$ & \cellcolor{headerpink}smECE$\downarrow$ & \cellcolor{headerblue}Pass@1$\uparrow$ & \cellcolor{headerpink}S-Acc.$\uparrow$ & \cellcolor{headerpink}smECE$\downarrow$ & \cellcolor{headerblue}Pass@1$\uparrow$ & \cellcolor{headerpink}S-Acc.$\uparrow$ & \cellcolor{headerpink}smECE$\downarrow$ \\
\midrule
\multicolumn{13}{c}{\textit{Fine-tuned model: Qwen3-8B-Base}} \\
\midrule
SFT & 79.1$_{1.9}$ & 92.8$_{1.7}$ & 14.7$_{2.0}$ & 38.9$_{0.5}$ & 49.5$_{0.6}$ & 7.5$_{0.3}$ & 50.4$_{0.4}$ & 72.0$_{1.3}$ & 16.8$_{1.9}$ & 56.1 & 71.4 & 13.0 \\
RCFT & 77.2$_{0.4}$ & 93.0$_{1.6}$ & 15.6$_{1.2}$ & 38.2$_{1.7}$ & 49.8$_{1.3}$ & 6.6$_{0.8}$ & 47.6$_{1.7}$ & 76.5$_{0.2}$ & 22.1$_{0.8}$ & 54.3 & 73.1 & 14.8 \\
LS & 78.4$_{1.7}$ & 93.7$_{0.5}$ & 14.6$_{0.2}$ & 38.9$_{0.5}$ & \underline{51.3}$_{0.7}$ & \textbf{4.9}$_{1.0}$ & 45.2$_{1.2}$ & 74.3$_{1.4}$ & 22.7$_{1.7}$ & 54.2 & 73.1 & 14.1 \\
BLoB & 75.5$_{0.8}$ & \underline{94.8}$_{0.3}$ & 16.5$_{1.0}$ & 38.6$_{1.2}$ & \underline{51.3}$_{0.7}$ & \underline{5.9}$_{0.4}$ & 49.4$_{0.6}$ & \underline{77.6}$_{0.9}$ & 21.4$_{0.7}$ & 54.5 & \underline{74.6} & 14.6 \\
Dr.GRPO & \underline{85.4}$_{1.9}$ & \textbf{95.6}$_{1.5}$ & \underline{8.3}$_{1.2}$ & \underline{41.4}$_{0.8}$ & 50.9$_{0.2}$ & 10.9$_{0.8}$ & \underline{56.0}$_{1.6}$ & 75.4$_{0.4}$ & \underline{12.6}$_{0.7}$ & \underline{60.9} & 74.0 & \underline{10.6} \\
S-EDL & \textbf{88.0}$_{1.8}$ & \textbf{95.6}$_{1.0}$ & \textbf{5.7}$_{1.6}$ & \textbf{45.6}$_{0.2}$ & \textbf{54.6}$_{0.4}$ & 8.6$_{1.1}$ & \textbf{71.2}$_{2.0}$ & \textbf{79.6}$_{0.8}$ & \textbf{3.3}$_{1.0}$ & \textbf{68.3} & \textbf{76.6} & \textbf{5.9} \\
\midrule
\multicolumn{13}{c}{\textit{Fine-tuned model: Ministral3-8B-Base}} \\
\midrule
SFT & 62.9$_{0.3}$ & \textbf{96.7}$_{0.2}$ & 23.4$_{0.8}$ & 31.6$_{1.3}$ & 43.7$_{0.6}$ & \underline{6.4}$_{1.1}$ & 24.1$_{1.6}$ & 48.6$_{0.6}$ & \textbf{15.5}$_{0.8}$ & 39.5 & 63.0 & 15.1 \\
RCFT & 72.5$_{0.8}$ & 88.6$_{0.4}$ & 14.2$_{0.7}$ & 33.7$_{0.6}$ & 43.2$_{0.7}$ & 6.7$_{0.8}$ & 28.7$_{0.8}$ & 54.5$_{0.7}$ & 17.9$_{0.9}$ & 45.0 & 62.1 & \underline{12.9} \\
LS & 55.1$_{3.1}$ & 80.9$_{2.2}$ & 22.4$_{1.2}$ & 26.0$_{0.8}$ & 41.2$_{0.6}$ & \textbf{6.2}$_{1.7}$ & 24.1$_{0.5}$ & 51.4$_{0.8}$ & 20.8$_{0.6}$ & 35.1 & 57.8 & 16.5 \\
BLoB & 69.9$_{0.8}$ & 89.3$_{1.0}$ & 18.0$_{1.4}$ & 34.4$_{0.9}$ & 44.9$_{0.9}$ & 6.5$_{0.6}$ & 29.2$_{1.0}$ & 55.0$_{0.8}$ & 18.0$_{1.4}$ & 44.5 & 63.1 & 14.2 \\
Dr.GRPO & \underline{83.0}$_{0.5}$ & 86.5$_{0.4}$ & \textbf{3.0}$_{0.7}$ & \underline{39.9}$_{0.5}$ & \underline{45.2}$_{0.6}$ & 21.1$_{0.5}$ & \underline{42.5}$_{0.8}$ & \underline{58.1}$_{0.9}$ & 16.6$_{1.0}$ & \underline{55.1} & \underline{63.3} & 13.6 \\
S-EDL & \textbf{87.7}$_{1.4}$ & \underline{94.1}$_{0.6}$ & \underline{5.0}$_{0.6}$ & \textbf{41.5}$_{1.0}$ & \textbf{51.8}$_{1.8}$ & 11.0$_{0.8}$ & \textbf{55.2}$_{0.6}$ & \textbf{76.5}$_{0.7}$ & \underline{15.8}$_{1.1}$ & \textbf{61.5} & \textbf{74.1} & \textbf{10.6} \\
\midrule
\multicolumn{13}{c}{\textit{Fine-tuned model: MiMo-7B-Base}} \\
\midrule
SFT & 63.3$_{1.6}$ & 87.1$_{0.8}$ & 22.3$_{1.0}$ & 27.8$_{0.5}$ & 42.9$_{0.6}$ & \underline{4.7}$_{0.1}$ & 22.0$_{0.9}$ & 50.9$_{0.3}$ & \underline{20.9}$_{0.9}$ & 37.7 & 60.3 & 16.0 \\
RCFT & 66.4$_{0.9}$ & 87.2$_{0.1}$ & 22.1$_{0.9}$ & 28.2$_{1.8}$ & \underline{44.6}$_{1.2}$ & 5.9$_{0.4}$ & 23.4$_{0.3}$ & 49.2$_{0.3}$ & 21.3$_{0.7}$ & 39.3 & 60.3 & 16.4 \\
LS & 64.2$_{1.6}$ & 86.9$_{0.8}$ & 22.0$_{0.7}$ & 26.8$_{1.8}$ & 42.6$_{1.8}$ & 4.9$_{1.6}$ & 21.2$_{1.8}$ & 49.3$_{0.8}$ & 22.7$_{0.5}$ & 37.4 & 59.6 & 16.5 \\
BLoB & 66.8$_{0.9}$ & 90.0$_{0.9}$ & 21.6$_{0.4}$ & 29.2$_{0.9}$ & 43.2$_{0.6}$ & \textbf{4.4}$_{0.9}$ & 21.8$_{0.9}$ & 51.0$_{1.2}$ & 22.0$_{0.8}$ & 39.3 & 61.4 & 16.0 \\
Dr.GRPO & \textbf{82.4}$_{0.8}$ & \underline{91.5}$_{0.3}$ & \textbf{8.6}$_{0.2}$ & \underline{35.2}$_{1.6}$ & \underline{44.6}$_{1.2}$ & 7.0$_{0.5}$ & \underline{31.4}$_{0.3}$ & \underline{59.0}$_{0.3}$ & 21.2$_{0.4}$ & \underline{49.7} & \underline{65.0} & \underline{12.3} \\
S-EDL & \underline{81.6}$_{0.7}$ & \textbf{92.6}$_{0.8}$ & \underline{11.0}$_{0.6}$ & \textbf{35.8}$_{0.9}$ & \textbf{47.6}$_{0.7}$ & 6.6$_{0.5}$ & \textbf{41.4}$_{0.6}$ & \textbf{67.0}$_{0.9}$ & \textbf{19.1}$_{0.6}$ & \textbf{52.9} & \textbf{69.1} & \textbf{12.2} \\
\bottomrule
\end{tabular}%
}
\end{table}

\subsection{Token probability profiles of generated solutions
}
\label{sec:token_prob}

\begin{figure}[t]
  \centering
  \includegraphics[width=\linewidth]{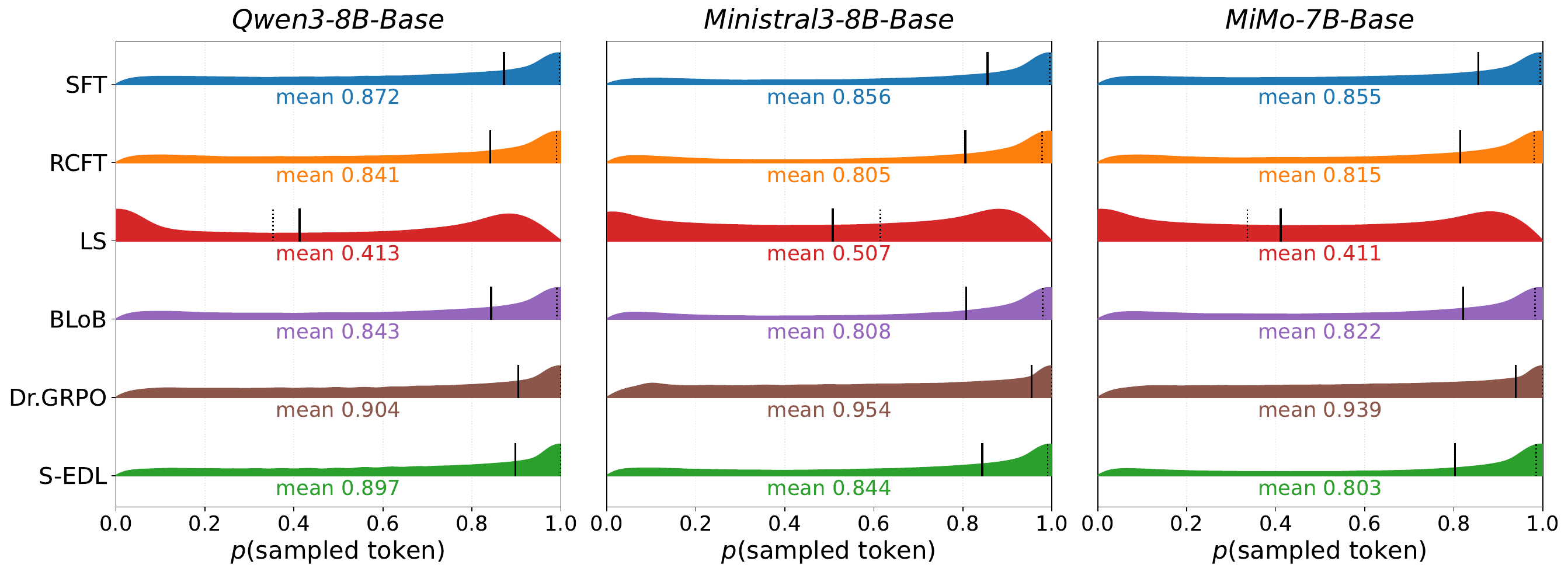}
  \caption{Violin plots of per-token probability profiles on MATH-500 generated answers. We record the probabilities of the sampled tokens. Solid and dotted black ticks mark the mean and median, respectively. Logarithm is applied to the violin widths to improve readability.}
\label{fig:token_probability}
\end{figure}

To investigate the source of calibration improvements in S-EDL, we analyze the sampled-token probability profiles of generated solutions. For each generated solution $z$ and decoded position $t$, we record $\pthetav(z_t \mid x, z_{<t})$ and pool these values across all token positions $t$ and samples $z$. 

\cref{fig:token_probability} and \cref{fig:token-prob-violin} (in \cref{sec:app:additional_per_token_prob}) illustrate three distinct mechanisms for altering token confidence on MATH-500. In \cref{fig:token_probability}, LS globally suppresses token probability, yielding the lowest mean $p(\text{sampled token})$ values of 0.413, 0.507, and 0.411 across the three models. Conversely, Dr.GRPO systematically sharpens distributions, reaching high-end means of 0.904, 0.954, and 0.939.

Unlike LS and Dr.GRPO, S-EDL preserves the SFT token probability profile. Across all twelve model-dataset pairs, the absolute difference between the S-EDL mean and the SFT mean remains below 0.06, and the sign of this minor difference varies across settings. Furthermore, S-EDL maintains a pooled token profile statistically indistinguishable from those of RCFT and BLoB, with means differing by at most 0.06, yet achieves average smECE values at least 2 points lower than all three. The calibration gain therefore cannot be explained by a coarse shift in token concentration. 
%Despite having similar pooled token probabilities, S-EDL achieves average smECE values at least 2 points lower than SFT, RCFT, and BLoB. This confirms that S-EDL's calibration gains originate from a targeted mechanism rather than a coarse shift in token concentration. 

This empirical behavior directly corroborates the semantic-level reallocation described by \cref{thm:conservation}. %S-EDL achieves semantic calibration by selectively routing probability mass across distinct semantic outcomes without imposing a fixed direction of change on local token probabilities.
S-EDL achieves semantic calibration by operating on the semantic-class marginal rather than uniformly on token distributions, routing probability mass between semantic outcomes without imposing a fixed direction of change on token probabilities. Pooled statistics cannot rule out position- or class-specific shifts, but the absence of global drift is informative: it isolates the calibration gain to a structural property of evidential loss rather than a side effect of token-level regularization.

\subsection{Ablation study}
\label{sec:ablation}

\cref{fig:ablation_beta_weight_qwen3_8b} shows how S-EDL responds to $\beta$ and $\lmseweight$. Across the four values of $\lmseweight$ spanning more than an order of magnitude, all three metrics improve over the SFT reference. As $\lmseweight$ grows, Pass@1 rises and smECE decreases mildly, at the cost of a small drop in S-Acc. The behavior of $\beta$ aligns with the gradient analysis in \cref{thm:ll_gradient}. Since the per-sample gradient of $\LMSE$ is scaled by $\beta - 1$, the extreme setting $\beta = 1.01$ attenuates the calibration signal by two orders of magnitude, and we include it as a stress test. For the remaining values $\beta \in \{1.1, 1.2, 1.5\}$, S-EDL improves over SFT on all three metrics, with $\beta = 1.2$ providing the best balance; $\beta = 1.5$ overweights the highest-likelihood samples and yields slightly worse calibration. We therefore adopt $\beta = 1.2$ and $\lmseweight = 5.0$ throughout all experiments. See \cref{sec:app:additional_ablation_beta_weight} for $\beta$ and $\lmseweight$ ablations in other settings. 

In \cref{fig:ablation_n_train}, $N_{\text{train}}$ varies from 4 to 32. Even at $N_{\text{train}} = 4$, S-EDL already improves over SFT on all three metrics. Increasing $N_{\text{train}}$ yields monotonic gains, with Pass@1 and S-Acc rising and smECE decreasing, consistent with \cref{thm:consistency}. 

\begin{figure}[t]
    \centering
    \begin{subfigure}[b]{0.32\textwidth}
        \centering
        \includegraphics[width=\linewidth]{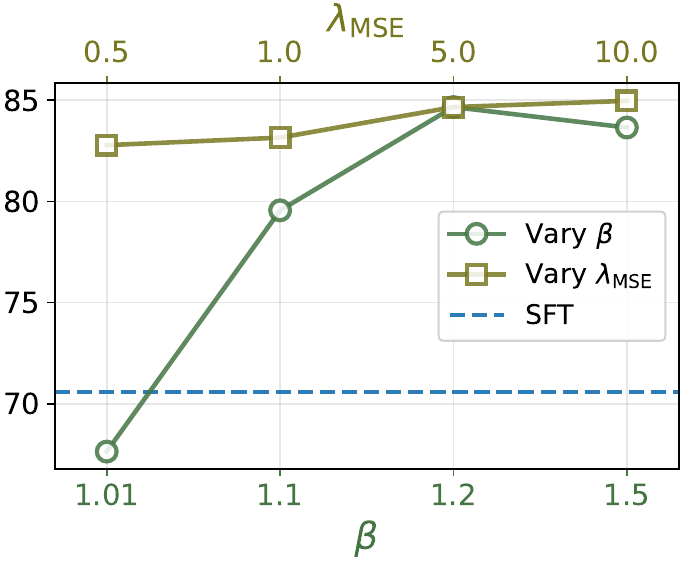}
        \caption{Pass@1 ($\uparrow$)}
        \label{fig:ablation_beta_pass1}
    \end{subfigure}
    \begin{subfigure}[b]{0.32\textwidth}
        \centering
        \includegraphics[width=\linewidth]{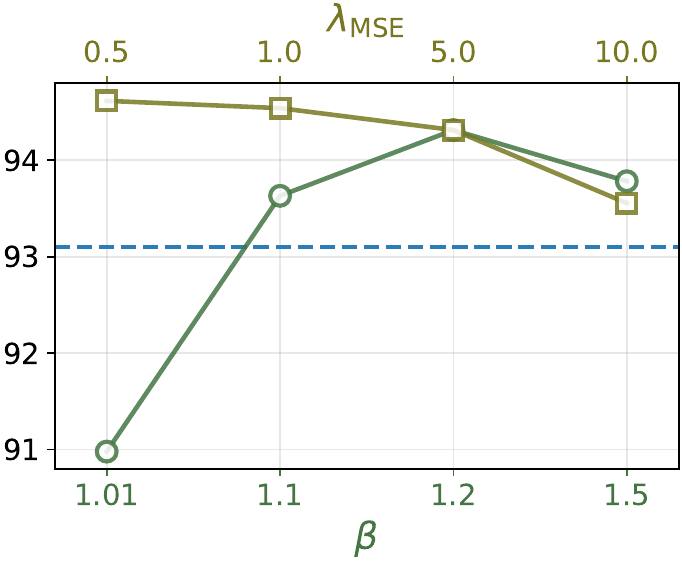}
        \caption{S-Acc. ($\uparrow$)}
        \label{fig:ablation_beta_sem_acc}
    \end{subfigure}
    \begin{subfigure}[b]{0.32\textwidth}
        \centering
        \includegraphics[width=\linewidth]{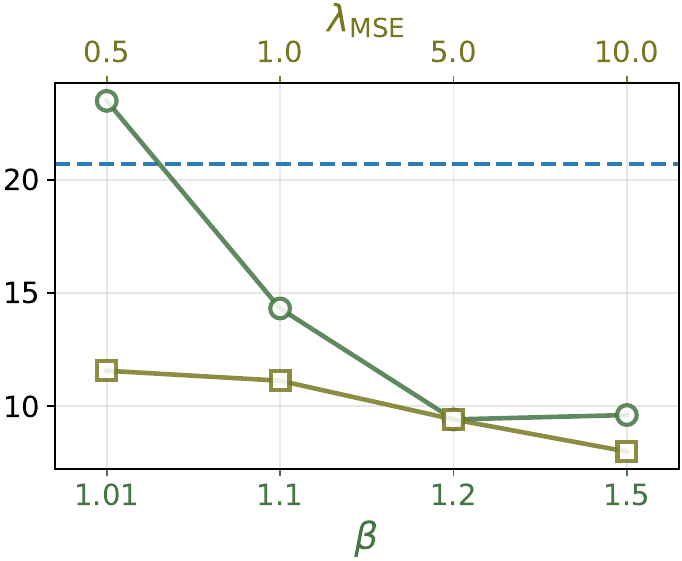}
        \caption{smECE ($\downarrow$)}
        \label{fig:ablation_beta_smece}
    \end{subfigure}
    \caption{Ablation study on $\beta$ and $\lambda_{\text{MSE}}$ with Qwen3-8B-Base trained and evaluated on GSM8K. The bottom $x$-axis varies $\beta$ with $\lambda_{\text{MSE}} = 5.0$ held fixed; the top $x$-axis varies $\lambda_{\text{MSE}}$ with $\beta = 1.2$ fixed.}
    \label{fig:ablation_beta_weight_qwen3_8b}
\end{figure}

\begin{figure}[t]
    \centering
    \begin{subfigure}[b]{0.32\textwidth}
        \centering
        \includegraphics[width=\linewidth]{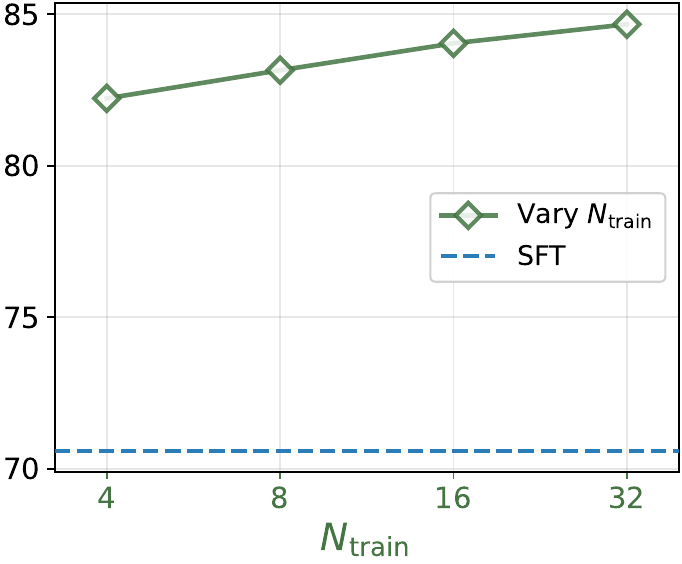}
        \caption{Pass@1 ($\uparrow$)}
        \label{fig:ablation_n_train_pass1}
    \end{subfigure}
    \begin{subfigure}[b]{0.32\textwidth}
        \centering
        \includegraphics[width=\linewidth]{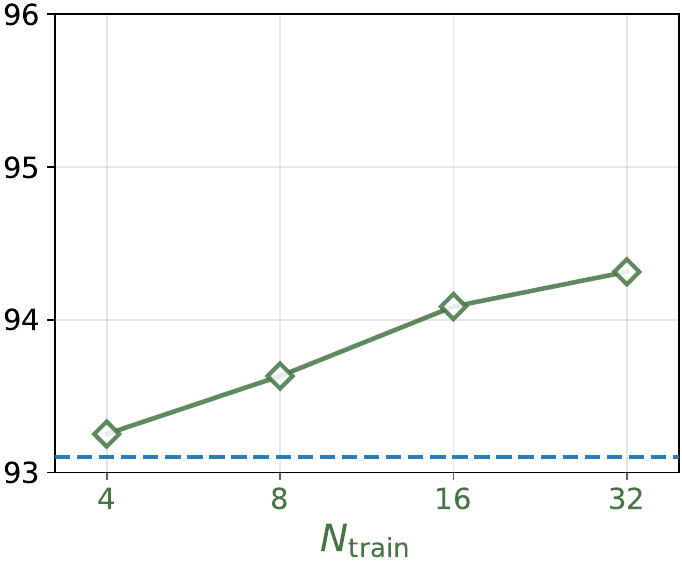}
        \caption{S-Acc. ($\uparrow$)}
        \label{fig:ablation_n_train_sem_acc}
    \end{subfigure}
    \begin{subfigure}[b]{0.32\textwidth}
        \centering
        \includegraphics[width=\linewidth]{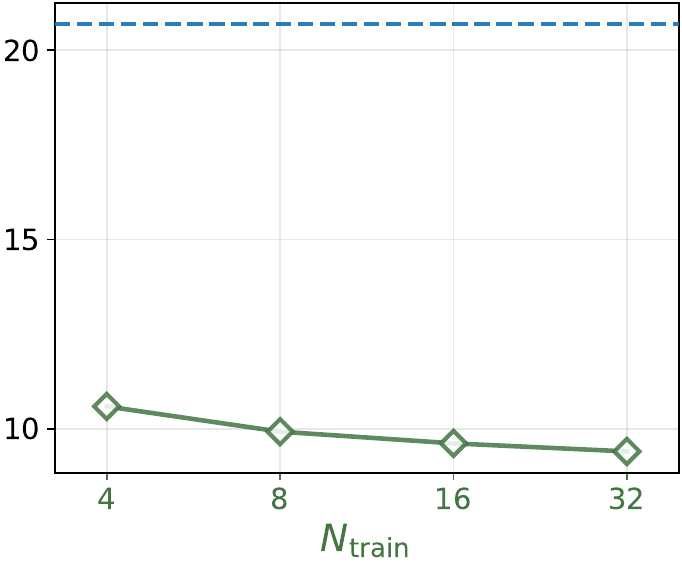}
        \caption{smECE ($\downarrow$)}
        \label{fig:ablation_n_train_smece}
    \end{subfigure}
    \caption{Ablation study on the number of Monte-Carlo samples $N_{\text{train}}$ with Qwen3-8B-Base trained and evaluated on GSM8K. The bottom $x$-axis varies $N_{\text{train}}$ with $\beta = 1.2$ and $\lambda_{\text{MSE}} = 5.0$ held fixed.}
    \label{fig:ablation_n_train}
\end{figure}

\subsection{Accelerating semantic clustering by a find-first algorithm}
The number of sampled sequences $N_{\text{train}}$ also affects training speed, since computing the EDL loss requires clustering these sequences into semantic classes.
To simplify the notation in the following analysis, we will use $N$ instead of $N_{\text{train}}$. A naive approach, which compares every pair of sequences, requires $O(N^2)$ oracle calls. By contrast, Dr.GRPO uses the verifier or LLM judge to score each sequence individually, requiring only $O(N)$ calls. 
To reduce this gap, we introduce a \emph{find-first clustering} algorithm, detailed in \cref{alg:find-first} in \cref{sec:app:find_first_algorithm}. 
Instead of exhaustive pairwise checks, the algorithm maintains a growing list of cluster representatives and compares each new sequence against them until it finds an equivalent one. If no match is found, it then creates a new cluster. 

As illustrated in \cref{fig:clustering_complexity}, clustering $N=50$ sampled sequences with the naive pairwise approach requires $N(N-1)/2 = 1225$ verifier calls per prompt. 
Across the evaluated models, find-first clustering reduces the average to approximately 65 calls on GSM8K and 165 calls on MATH, corresponding to roughly $1.3N$ and $3.3N$ calls, respectively.
This cuts the computational burden to a small fraction of the $O(N^2)$ worst case, bringing the practical cost much closer to $O(N)$. 

For general tasks that rely on an LLM judge, each sequence can instead be matched against all current cluster representatives in a single classification call.
This implementation requires $O(N)$ judge calls per question prompt, matching the $O(N)$ calls of Dr.GRPO.
Decision models such as Jev~\citep{typesafe2026docs} could be explored to further accelerate this classification step.
Across both verifier-based and judge-based settings, these implementations substantially narrows the gap in oracle calls between S-EDL and other sequence-level methods such as Dr.GRPO.

\begin{figure}[t]
    \centering
    \begin{subfigure}[b]{0.49\textwidth}
        \centering
        \includegraphics[width=\linewidth]{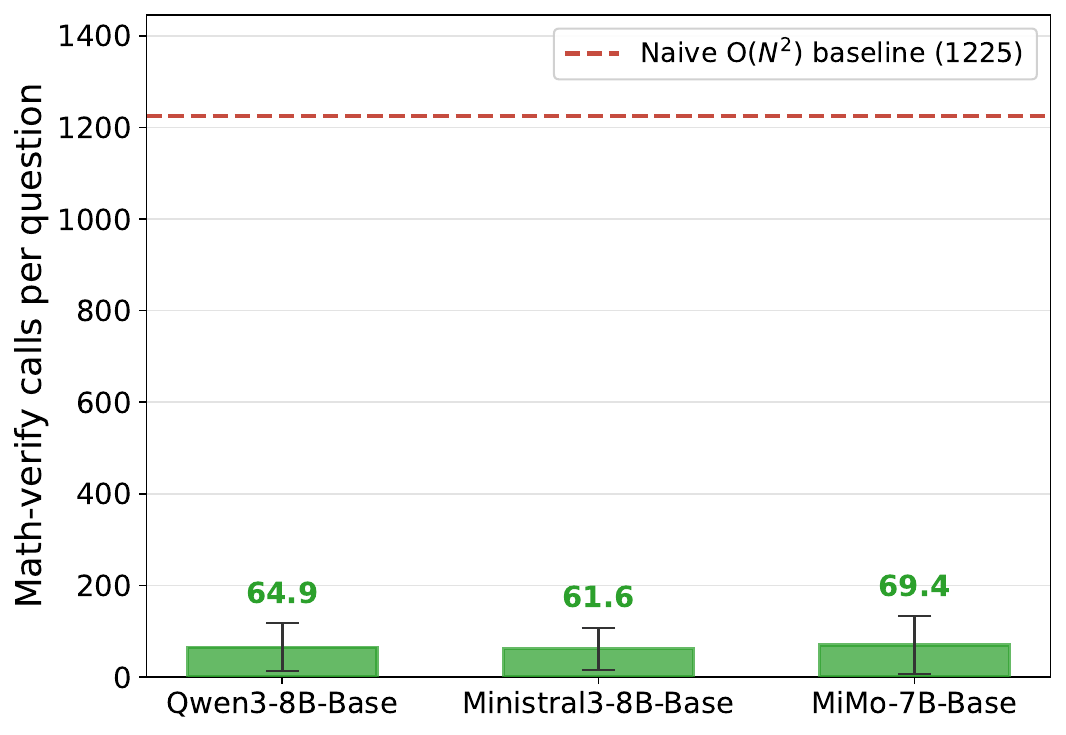}
        \caption{GSM8K}
    \end{subfigure}
    \begin{subfigure}[b]{0.49\textwidth}
        \centering
        \includegraphics[width=\linewidth]{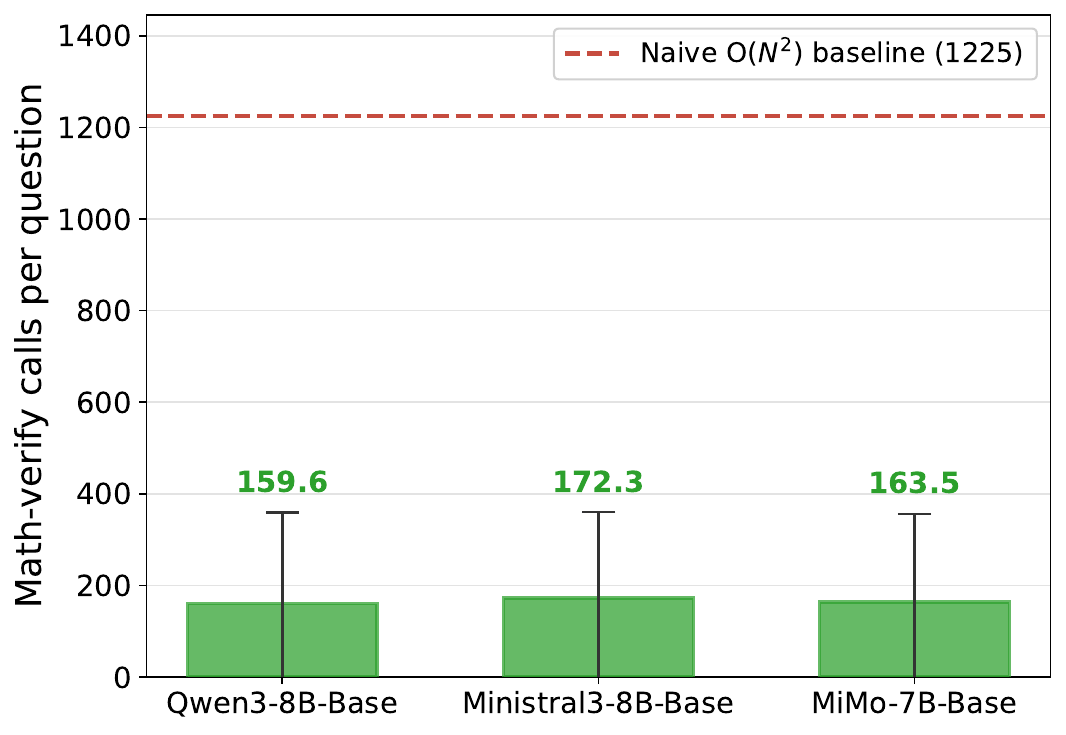}
        \caption{MATH}
    \end{subfigure}
    \caption{Empirical cost of find-first semantic clustering on math tasks. Bars show the mean number of symbolic verifier calls per question when clustering $N=50$ generated sequences, with error bars denoting one standard deviation. The dashed red line marks the naive pairwise baseline of $N(N-1)/2 = 1225$ calls per question. The find-first algorithm achieves a substantial reduction in computational cost, bringing the empirical cost much closer to $O(N)$. Specifically, it averages roughly $1.3N$ calls on GSM8K and $3.3N$ calls on MATH. The reduction is smaller on MATH, where harder problems produce more distinct semantic classes and therefore require more comparisons.}
    \label{fig:clustering_complexity}
\end{figure}

%% file: sections/6_conclusion.tex
\section{Conclusion}
\label{sec:conclusion}

We presented S-EDL, a method that constructs a Dirichlet prior over prompt-specific semantic classes from a model's own sequence likelihoods, yielding a differentiable evidential objective for fine-tuning LLMs toward semantic calibration. Our gradient analysis shows that this objective conservatively reallocates probability mass across semantic outcomes, raising the likelihoods of sequences in the ground-truth class while suppressing those of the most overconfident incorrect alternative. Empirically, this reallocation produces simultaneous gains in accuracy and calibration. More broadly, our results suggest that grounding calibration objectives in a model's own generative distribution offers a principled path toward trustworthy open-ended LLMs.

\textbf{Limitations.}\quad While our find-first algorithm mentioned above substantially accelerates semantic grouping, the associated sampling and clustering steps entail a moderate training overhead compared to standard SFT. Additionally, for general tasks that necessitate an LLM judge, this requirement imposes unavoidable memory overhead and potential noise in clustering outcomes. To streamline this process, future work could bypass external verifiers or LLMs entirely by leveraging the hidden representations of generated sequences to estimate semantic dispersion, conceptually similar to methods like EigenScore~\citep{chen2024inside}. Finally, we have not yet tested S-EDL on long-context tasks or domains where semantic equivalence is intrinsically difficult to define, such as code generation. Adapting S-EDL to these challenging scenarios remains an important direction for future exploration.

%% file: sections/7_appendix.tex
\appendix

\crefalias{section}{appendix}
\crefalias{subsection}{appendix}
\crefalias{subsubsection}{appendix}

%---SECTION---%
\section{Proofs}
\subsection{Proof of \cref{thm:consistency}} \label{sec:app:proof_consistency}

\begin{proof}[Proof of \cref{thm:consistency}]
Apply the Strong Law of Large Numbers separately to the numerator and denominator of \cref{eq:estimator}. The numerator $\tfrac{1}{N} U_s = \tfrac{1}{N} \sum_{i=1}^{N} \ui\, \mathbbm{1} \lbrace s^{(i)} = s \rbrace$ is a sample average of i.i.d.\ terms under $\zi \sim \pthetav(\cdot | x)$, so almost surely
\begin{align}
\lim_{N \to \infty} \tfrac{1}{N} U_s
  &= \E_{z \sim \pthetav(\cdot | x)} \!\left[ \pthetav(z | x)^{\beta - 1}\, \mathbbm{1} \lbrace B(x, z) = s \rbrace \right] \\
  &= \sum_{z \in \Vocab^{L}} \pthetav(z | x) \cdot \pthetav(z | x)^{\beta - 1}\, \mathbbm{1} \lbrace B(x, z) = s \rbrace \\
  &= \sum_{z\,:\,B(x,z) = s} \pthetav(z | x)^\beta.
\end{align}
An identical argument applied to $\tfrac{1}{N} \sum_{r \in \mathcal{S}} U_r$ yields the partition function $\sum_{z' \in \Vocab^{L}} \pthetav(z' | x)^\beta$. The continuous mapping theorem then identifies the ratio of the two limits with $\pi_\beta(s)$ in \cref{eq:tempered}.
\end{proof}

\subsection{Proof of \cref{thm:ll_gradient}} 
\label{sec:app:proof_ll_gradient}

\begin{proof}[Proof of \cref{thm:ll_gradient}]
The forward pass composes four maps:
\[
    \elli \,\xrightarrow{\exp((\beta - 1)\,\cdot)}\, \ui \,\xrightarrow{\text{class sum}}\, U_s \,\xrightarrow{\text{normalize}}\, \hat\pi_\beta(s) \,\xrightarrow{\text{EDL loss}}\, \LMSE.
\]
We establish $d_s = \partial \LMSE / \partial \hat{\pi}_\beta(s)$ first, then backpropagate along the remaining three links.

\emph{Step 1: derivative w.r.t. $\hat\pi_\beta(s)$.}
Since $\alpha_s = N \hat\pi_\beta(s) + 1$, we have $\partial \alpha_s / \partial \hat\pi_\beta(s) = N$, and the Dirichlet strength $\alz = N + |\mathcal{S}|$ is constant in $\hat\pi_\beta$. 
Let $q_s := \mathbbm{1} \lbrace s = s^* \rbrace$ be the binary indicator of the ground truth class.
Differentiating the closed-form MSE in \cref{eq:mse} gives
\begin{equation}
\begin{aligned}
    \frac{\partial \LMSE}{\partial \hat\pi_\beta(s)}
    &= \underbrace{\frac{\partial \alpha_s}{\partial \hat\pi_\beta(s)}}_{= N} \cdot \frac{\partial \LMSE}{\partial \alpha_s} \\
    &= N \left[ 2 \left( \frac{\alpha_s}{\alz} - q_s \right) \frac{1}{\alz} + \frac{\alz - 2\alpha_s}{\alz^2 (\alz + 1)} \right] \\
    &= \frac{N}{\alz^2 (\alz + 1)} \left[ 2(\alpha_s - q_s \alz)(\alz + 1) + \alz - 2\alpha_s \right] \\
    &= \frac{N}{\alz (\alz + 1)} \left[ 2\alpha_s - 2 q_s (\alz + 1) + 1 \right].
\end{aligned}
\end{equation}
Substituting $\alz = N + |\mathcal{S}|$ yields 

\begin{equation}\label{eq:ds}
    d_s = \frac{N \left[\, 2\alpha_s - 2(N + |\mathcal{S}| + 1)\, q_s + 1 \,\right]}{(N + |\mathcal{S}|)(N + |\mathcal{S}| + 1)},
\end{equation}

\emph{Step 2: backpropagation from $\hat\pi_\beta$ to $\elli$.}
Let $U := \sum_{s \in \mathcal{S}} U_s = \sum_{i=1}^{N} \ui$. Then:
\begin{enumerate}[leftmargin=*]
\item $\displaystyle \frac{\partial \hat\pi_\beta(s)}{\partial U_r} = \frac{1}{U} \left( \delta_{sr} - \hat\pi_\beta(s) \right)$, by direct differentiation of $U_s / U$.

\item $\displaystyle \frac{\partial \LMSE}{\partial U_r} = \sum_{s \in \mathcal{S}} d_s \cdot \frac{1}{U} \left( \delta_{sr} - \hat\pi_\beta(s) \right) = \frac{1}{U} \left( d_r - \bar d \right)$.

\item $\partial U_s / \partial \ui = \mathbbm{1} \lbrace s^{(i)} = s \rbrace$, hence $\partial \LMSE / \partial \ui = (1/U) \left( d_{s^{(i)}} - \bar d \right)$.

\item $\partial \ui / \partial \elli = (\beta - 1)\, \ui$, so
\begin{equation*}
    \frac{\partial \LMSE}{\partial \elli} = \frac{(\beta - 1)\, \ui}{U} \left( d_{s^{(i)}} - \bar d \right) = (\beta - 1)\, w^{(i)} \left( d_{s^{(i)}} - \bar d \right),
\end{equation*}
\end{enumerate}
which is \cref{eq:dl_dell}.
\end{proof}

\subsection{Proof of \cref{thm:sign_of_grad}}
\label{sec:app:proof_sign_of_grad}

\begin{proof}[Proof of \cref{thm:sign_of_grad}]
Substituting $\alpha_s = N\hat\pi_\beta(s) + 1$ into \cref{eq:ds} reveals that the class-level derivatives $d_s$ possess opposite signs for the ground-truth class and the incorrect classes. Specifically, for the ground-truth class $s^*$, we have
\begin{equation*}
    d_{s^*} \propto 2N(\hat\pi_\beta(s^*) - 1) - 2|\mathcal{S}| + 1 < 0,
\end{equation*}
and for any incorrect class $s \neq s^*$, we have
\begin{equation*}
    d_s \propto 2N\hat\pi_\beta(s) + 3 > 0.
\end{equation*}
Consequently, $d_{s^*} < d_s$ holds for every $s \neq s^*$. Furthermore, since $d_s$ strictly increases with respect to $\hat\pi_\beta(s)$ among the incorrect classes, the most overconfident incorrect class defined by $s^- := \arg\max_{s \neq s^*} \hat\pi_\beta(s)$
attains the maximum value of $d_s$ among the incorrect classes, and therefore among all classes.
From \cref{thm:ll_gradient}, the gradient with respect to the sample log-likelihood is given by
\begin{equation*}
    \frac{\partial \LMSE}{\partial \ell^{(i)}} = (\beta - 1)w^{(i)}(d_{s^{(i)}} - \bar d), \qquad \text{where } \bar d = \sum_s \hat\pi_\beta(s)d_s.
\end{equation*}
Given that $\beta > 1$ and the normalized sample weight $w^{(i)} > 0$, the sign of the gradient is entirely determined by the difference $d_{s^{(i)}} - \bar d$.
Because $\bar d$ is a convex combination of all class-level derivatives and $d_{s^*}$ represents the global minimum among them, it follows naturally that $d_{s^*} - \bar d \leq 0$.
Therefore, for any sample belonging to the ground-truth class ($s^{(i)} = s^*$), the gradient satisfies $\frac{\partial \LMSE}{\partial \ell^{(i)}} \leq 0.$
This inequality is strict provided the semantic distribution has not fully collapsed to the ground-truth class ($\hat\pi_\beta(s^*) < 1$), because the convex combination $\bar d$ will include positive weights from the strictly larger derivatives of the incorrect classes.

Similarly, because $d_{s^-}$ is the global maximum among the class-level derivatives, we have $d_{s^-} - \bar d \geq 0$.
Thus, for any sample mapped to the most overconfident incorrect class ($s^{(i)} = s^-$), the gradient satisfies $\frac{\partial \LMSE}{\partial \ell^{(i)}} \geq 0$.
This inequality becomes strict whenever $\hat\pi_\beta(s^*) > 0$, as the convex combination $\bar d$ then incorporates a positive weight from the strictly smaller derivative $d_{s^*} < d_{s^-}$. This concludes the proof.
\end{proof}

\subsection{Proof of \cref{thm:conservation}}
\label{sec:app:proof_conservation}

\begin{proof}[Proof of \cref{thm:conservation}]
Using Theorem~1 and $\sum_{j:s^{(j)}=s} w^{(j)}=\hat\pi_\beta(s)$, we have
\[
\sum_{i=1}^{N} \frac{\partial \LMSE}{\partial \elli}
=
(\beta-1)\sum_{i=1}^{N} w^{(i)}(d_{s^{(i)}}-\bar d)
=
(\beta-1)\left(\sum_s \hat\pi_\beta(s)d_s-\bar d\sum_{i=1}^{N}w^{(i)}\right)
=0.
\]
The last equality follows from $\bar d=\sum_s \hat\pi_\beta(s)d_s$ and $\sum_{i=1}^{N}w^{(i)}=1$.
\end{proof}

\subsection{Relation to GRPO-like updates}
\label{sec:app:grpo_connection}

In this subsection, we show that S-EDL gradients can also be interpreted as GRPO-like updates.
Before detailing this connection, we note that the SFT term in $\LSEDL$ is not essential to the formulation of S-EDL.
In principle, the model can be trained using $\LMSE$ alone with $\beta>1$.
Its practical role is to provide dense token-level supervision early in training, when correct completions are rare and the positive signal from sampled semantic classes can be weak, as in sparse-reward learning.
Therefore, the following analysis omits the SFT term and focuses on $\LMSE$.

For a fixed set of sampled sequences and semantic assignments, \cref{eq:dl_dell} gives S-EDL a group-relative update structure similar to GRPO~\citep{shao2024deepseekmath}.
The key idea is to interpret the \emph{negative} gradient $-\partial\LMSE/\partial\elli$ as a reward signal for each sampled sequence.
Define the uncentered class reward $r_s := -d_s$ and its weighted mean $\bar r := \sum_{s\in\mathcal{S}}\hat\pi_\beta(s)r_s = -\bar d$.
The negative gradient with respect to each sample log-likelihood then acts as an advantage coefficient:
\begin{equation}\label{eq:sedl_advantage}
    A^{(i)} := -\frac{\partial \LMSE}{\partial \elli}
    = (\beta - 1)w^{(i)}\bigl(r_{s^{(i)}}-\bar r\bigr),
\end{equation}
and then we have
\begin{equation}
    -\nabla_{\thetav}\LMSE = \sum_{i=1}^{N} A^{(i)}\nabla_{\thetav}\elli.
\end{equation}
Here, the group comparison is over the semantic classes associated with the same prompt.
The class derivative $d_{s^{(i)}}$ is centered by $\bar d$, and the centered signal is modulated by $(\beta-1)w^{(i)}$, with the sign reversed for likelihood ascent.
Original GRPO centers rewards over sampled completions using their arithmetic mean and divides by a reward standard deviation shared by the group.
S-EDL instead uses a semantic-probability-weighted mean and a sample-dependent modulation factor, so completions in the same class share the class reward but can receive different update weights.

This interpretation allows S-EDL to be incorporated into proximal policy optimization (PPO) implementations~\citep{schulman2017proximal}, with clipping and micro-batch updates as in GRPO variants.
This can improve the sample efficiency and stability of S-EDL training.

%---END SECTION---%

% %---SECTION---%
% \section{Case studies on MATH}
% \label{app:case-studies}

% We illustrate the qualitative effect of the EDL objective with two
% representative cases drawn from the MATH-500 evaluation. For each case we
% compare the baseline (3-epoch SFT) against EDL fine-tuned from the
% 1-epoch SFT checkpoint with $w{=}5.0$, $\beta{=}1.2$. Both models are
% sampled 50 times at $T{=}0.7$ and grouped into semantic clusters by
% \textsc{math\_verify}+judge equivalence; cluster sizes and the modal cluster
% ($\bigstar$) summarise the predictive distribution.
% %---END SECTION---%

%---SECTION---%
\section{Implementation details}\label{sec:app:implementation_details}

\subsection{Construction of semantic class set and handling unobserved ground-truth classes}
\label{sec:unobserved_gt}
Previous work on semantic calibration \citep{nakkiran2026trained} establishes evaluation metrics based entirely on the sampled distribution. When adapting this purely sample-based perspective into a training objective, if the correct semantic class is absent from the generated samples, the evidential loss $\LMSE$ can still be computed, as the label vector $\qv$ then simply consists of all zero entries. Consequently, the gradient only suppresses the incorrect overconfident completions, but lacks the effect of increasing the likelihoods of the correct completions.

To remedy this and ensure the gradient direction toward the ground-truth class established in \cref{thm:sign_of_grad} remains available, S-EDL \emph{conditionally} introduces $s^*$ into the semantic class set $\mathcal{S}_x$ strictly when it is absent from the generated samples. This guarantees that we do not artificially inflate the sample count or alter the distribution when the model naturally generates the correct answer. We apply this conditional handling as follows to maximize consistency across the training and evaluation pipelines.

\begin{itemize}[leftmargin=1.5em]
\item When the ground-truth class is already present in the sampled generations, we use the sampled set directly without any modifications during both training and evaluation.
\item When the ground-truth class is unobserved during training, we explicitly include it by feeding the ground-truth sequence $y$ to the model to compute its log-likelihood and applying a penalty of $-4.6$. This penalty scales down its probability mass to reflect that the sequence was not naturally sampled, while still preserving the gradient path for optimization.
\item When the ground-truth class is unobserved during evaluation, we similarly include it but strictly set its sampled count to $0$, ensuring the evaluation accurately reflects the true generative distribution of the model.
\end{itemize}

Deterministically adding the ground-truth completion $y$ means the estimator is no longer a pure Monte Carlo estimator based solely on samples from $\pthetav$. Instead, it creates a hybrid objective that mixes the generative distribution of the model with a supervised finite-candidate signal. We accept this theoretical trade-off to maintain a robust and differentiable training path.

\subsection{Find-first algorithm for semantic clustering}\label{sec:app:find_first_algorithm}

\begin{algorithm}[t]
\caption{Find-first semantic clustering}
\label{alg:find-first}
\begin{algorithmic}[1]
\Require generated sequences $z^{(1)},\dots,z^{(N)}$; equivalence oracle $\textsc{Match}$;
         string similarity function $\sigma(\cdot,\cdot)$
\Ensure clusters $\mathcal{C}$, each a pair $(r,M)$ of representative answer and member set
\State $\mathcal{C} \gets \langle \ \rangle$
\For{$i = 1, \dots, N$}
    \State $K \gets |\mathcal{C}|$
    \State $\pi \gets \arg\!\operatorname{sort}_{k\in\{1,\dots,K\}} \big(-\sigma(z^{(i)}, r_k)\big)$
        \Comment{order reps by descending similarity to $z^{(i)}$}
    \State $j \gets \textsc{Match} \big(z^{(i)},\; [r_{\pi(1)}, \dots, r_{\pi(K)}]\big)$
        \Comment{position in list, or $\bot$ if none equivalent}
    \If{$j \neq \bot$}
        \State $M_{\pi(j)} \gets M_{\pi(j)} \cup \{z^{(i)}\}$
            \Comment{assign $z^{(i)}$ to the matched cluster}
    \Else
        \State append $(r, M) \gets (z^{(i)}, \{z^{(i)}\})$ to $\mathcal{C}$
            \Comment{open a new cluster}
    \EndIf
\EndFor
\State \Return $\mathcal{C}$
\end{algorithmic}
\end{algorithm}

Sequence-level training methods like Dr.GRPO require comparing $N$ sampled sequences against the ground truth answer, resulting in an $O(N)$ computational complexity per prompt. By contrast, computing the evidential loss in S-EDL requires grouping the generated sequences into semantic classes. A naive pairwise comparison between all sequences scales as $O(N^2)$, imposing a substantial computational bottleneck.

To address this scaling gap, we introduce a find-first semantic clustering algorithm as shown in \cref{alg:find-first}. Instead of brute-force pairwise comparison, the algorithm maintains a growing list of cluster representatives and dynamically evaluates each new sequence against them. To optimize the comparison order, we sort the existing representatives by descending string similarity to the new sequence using \texttt{rapidfuzz}~\citep{bachmann2025rapidfuzz}. The equivalence oracle $\textsc{Match}$ is then instantiated in two ways depending on the evaluation domain.

\textbf{Mathematical tasks}\qquad $\textsc{Match}(z, [r_1, \dots, r_K])$ evaluates equivalence using a symbolic verifier (e.g. \texttt{math-verify}\citep{kydlicek2025mathverify}). It sequentially checks the candidate against the sorted representatives and returns the first index $j$ where $z \equiv r_j$, returning $\bot$ if no match exists. The string similarity ordering heavily biases the search toward early termination. 

\textbf{General tasks}\qquad For general subject tasks that rely on an LLM as a judge, $\textsc{Match}$ replaces iterative pairwise checks with a single listwise classification prompt. The judge evaluates the sequence against all current cluster representatives in one forward pass, directly returning the matched index in $\{1, \dots, K\}$ or $\bot$. The system is shown in \cref{sec:app:prompt_llm_as_judge}. In this context, the string similarity ordering ensures the most plausible candidate appears first in the prompt context. By grouping the comparisons into a single listwise operation, the temporal complexity of semantic clustering is reduced to exactly $O(N)$, identically matching the scaling of Dr.GRPO.

\subsection{Dataset construction}
\label{sec:app:dataset_construction}

\paragraph{GSM8K} We use the original training and test splits from the official release, along with the provided solutions and chain-of-thought reasoning steps.

\paragraph{MATH} We utilize the official training split and the associated chain-of-thought solutions. For evaluation, we employ the widely adopted MATH-500 test split.

\paragraph{ChemistryQA} We use the official training and test splits. To construct the training data, we prompt the GPT-OSS-120B model \citep{agarwal2025gpt} to generate chain-of-thought solutions, retaining only those that yield the correct final answer. This filtering process results in 3062 training samples. For evaluation, we use the complete test set and assess performance directly against the provided ground-truth answers.

\paragraph{FormulaReasoning} We similarly use the official training and test splits. We generate chain-of-thought solutions using the GPT-OSS-120B model and filter for correctness, which yields 4096 training samples. We evaluate on the complete test set using the official ground-truth answers. During the cross-domain evaluation, we observed that models fine-tuned on the MATH dataset experience a severe accuracy degradation for all studied fine-tuning methods. This drop makes calibration performance indistinguishable across the fine-tuning methods, as most samples exhibit uniformly low semantic confidence and low accuracy across all methods. We identified that the models struggle to recall the necessary physical laws. Consequently, we manually augment the testing prompts with the relevant physical laws as hints. These hint-augmented prompts are applied exclusively during the cross-domain evaluation. Furthermore, these hint-augmented prompts are used across evaluation of all fine-tuning methods to ensure fairness.

Examples of the training prompts are provided in \cref{sec:app:examples_training_prompts}.

\begin{table}[h!]
\centering
\caption{EDL fine-tuning hyperparameters for Qwen3-8B-Base across the four evaluation datasets.}
\label{tab:hparams_qwen3_8b}
\small
\resizebox{\textwidth}{!}{%
\begin{tabular}{lcccc}
\toprule
Hyperparameter & GSM8K & MATH & ChemistryQA & FormulaReasoning \\
\midrule
\multicolumn{5}{l}{\textit{Model \& LoRA}} \\
Model \texttt{max\_seq\_length}  & 2048 & 4096 & 4096 & 4096 \\
LoRA $r$ / $\alpha$ / dropout    & 16 / 32 / 0.0 & 16 / 32 / 0.0 & 16 / 32 / 0.0 & 16 / 32 / 0.0 \\
LoRA target modules              & Q,K,V,O,gate,up,down & Q,K,V,O,gate,up,down & Q,K,V,O,gate,up,down & Q,K,V,O,gate,up,down \\
LoRA bias                        & none & none & none & none \\
Precision                        & bf16 & bf16 & bf16 & bf16 \\
\midrule
\multicolumn{5}{l}{\textit{Optimizer (cosine schedule)}} \\
Optimizer                        & AdamW & AdamW & AdamW & AdamW \\
LR scheduler                     & cosine & cosine & cosine & cosine \\
Warmup ratio                     & 0.06 & 0.06 & 0.06 & 0.06 \\
Max grad norm                    & 1.0 & 1.0 & 1.0 & 1.0 \\
Gradient checkpointing           & enabled & enabled & enabled & enabled \\
\midrule
\multicolumn{5}{l}{\textit{EDL training}} \\
EDL initialization               & base & base & base & base \\
\texttt{num\_train\_epochs}      & 1 & 1 & 1 & 1 \\
Per-device batch size            & 4 & 4 & 4 & 4 \\
Peak LR                          & $2\!\times\!10^{-4}$ & $2\!\times\!10^{-4}$ & $2\!\times\!10^{-4}$ & $2\!\times\!10^{-4}$ \\
\texttt{edl\_weight\_start\_step} & 0.5 & 0.5 & 0.5 & 0.5 \\
EDL $\beta$                      & 1.2 & 1.2 & 1.2 & 1.2 \\
EDL weight $\lmseweight$         & 5.0 & 5.0 & 5.0 & 5.0 \\
$N_{\text{train}}$               & 32 & 8 & 8 & 8 \\
Sampling \texttt{max\_tokens}    & 256 & 768 & 768 & 768 \\
Sampling temperature             & 1.0 & 1.0 & 1.0 & 1.0 \\
Sampling \texttt{top\_p}         & 0.9 & 0.9 & 0.9 & 0.9 \\
\midrule
\multicolumn{5}{l}{\textit{Evaluation}} \\
$N_{\text{eval}}$                & 50 & 50 & 50 & 50 \\
Temperature                      & 0.8 & 0.8 & 0.8 & 0.8 \\
\texttt{top\_p}                  & 0.9 & 0.9 & 0.9 & 0.9 \\
\texttt{max\_tokens}             & 1024 & 2048 & 2048 & 2048 \\
\midrule
\multicolumn{5}{p{1.3\linewidth}}{\footnotesize\emph{Note.} All EDL runs use a cosine LR schedule with \texttt{warmup\_ratio}~$=0.06$ over the entire \texttt{num\_train\_epochs}; the EDL loss is enabled starting at fraction \texttt{edl\_weight\_start\_step} of total steps. Where \emph{EDL initialization} is \emph{base}, the cosine schedule covers the full run and EDL begins partway through the decay; where \emph{EDL initialization} is \emph{SFT ckpt (epoch=1)}, the run is short and the cosine schedule covers only the EDL phase.} \\
\bottomrule
\end{tabular}%
}
\end{table}

\begin{table}[h!]
\centering
\caption{EDL fine-tuning hyperparameters for Ministral-3-8B-Base across the four evaluation datasets.}
\label{tab:hparams_ministral3_8b}
\small
\resizebox{\textwidth}{!}{%
\begin{tabular}{lcccc}
\toprule
Hyperparameter & GSM8K & MATH & ChemistryQA & FormulaReasoning \\
\midrule
\multicolumn{5}{l}{\textit{Model \& LoRA}} \\
Model \texttt{max\_seq\_length}  & 2048 & 4096 & 4096 & 4096 \\
LoRA $r$ / $\alpha$ / dropout    & 16 / 32 / 0.0 & 16 / 32 / 0.0 & 16 / 32 / 0.0 & 16 / 32 / 0.0 \\
LoRA target modules              & Q,K,V,O,gate,up,down & Q,K,V,O,gate,up,down & Q,K,V,O,gate,up,down & Q,K,V,O,gate,up,down \\
LoRA bias                        & none & none & none & none \\
Precision                        & bf16 & bf16 & bf16 & bf16 \\
\midrule
\multicolumn{5}{l}{\textit{Optimizer (cosine schedule)}} \\
Optimizer                        & AdamW & AdamW & AdamW & AdamW \\
LR scheduler                     & cosine & cosine & cosine & cosine \\
Warmup ratio                     & 0.06 & 0.06 & 0.06 & 0.06 \\
Max grad norm                    & 1.0 & 1.0 & 1.0 & 1.0 \\
Gradient checkpointing           & enabled & enabled & enabled & enabled \\
\midrule
\multicolumn{5}{l}{\textit{EDL training}} \\
EDL initialization               & base & SFT ckpt (epoch=1) & SFT ckpt (epoch=1) & SFT ckpt (epoch=1) \\
\texttt{num\_train\_epochs}      & 2.5 & 0.3 & 0.3 & 0.5 \\
Per-device batch size            & 4 & 4 & 4 & 4 \\
Peak LR                          & $2\!\times\!10^{-4}$ & $5\!\times\!10^{-5}$ & $5\!\times\!10^{-5}$ & $5\!\times\!10^{-5}$ \\
\texttt{edl\_weight\_start\_step} & 0.8 & 0.0 & 0.0 & 0.0 \\
EDL $\beta$                      & 1.2 & 1.2 & 1.2 & 1.2 \\
EDL weight $\lmseweight$         & 5.0 & 5.0 & 5.0 & 5.0 \\
$N_{\text{train}}$               & 32 & 12 & 12 & 8 \\
Sampling \texttt{max\_tokens}    & 256 & 512 & 512 & 512 \\
Sampling temperature             & 1.0 & 0.5 & 0.5 & 1.0 \\
Sampling \texttt{top\_p}         & 0.9 & 0.9 & 0.9 & 0.9 \\
\midrule
\multicolumn{5}{l}{\textit{Evaluation}} \\
$N_{\text{eval}}$                & 50 & 50 & 50 & 50 \\
Temperature                      & 0.6 & 0.6 & 0.6 & 0.6 \\
\texttt{top\_p}                  & 0.9 & 0.9 & 0.9 & 0.9 \\
\texttt{max\_tokens}             & 1024 & 2048 & 2048 & 2048 \\
\midrule
\multicolumn{5}{p{1.3\linewidth}}{\footnotesize\emph{Note.} All EDL runs use a cosine LR schedule with \texttt{warmup\_ratio}~$=0.06$ over the entire \texttt{num\_train\_epochs}; the EDL loss is enabled starting at fraction \texttt{edl\_weight\_start\_step} of total steps. Where \emph{EDL initialization} is \emph{base}, the cosine schedule covers the full run and EDL begins partway through the decay; where \emph{EDL initialization} is \emph{SFT ckpt (epoch=1)}, the run is short and the cosine schedule covers only the EDL phase.} \\
\bottomrule
\end{tabular}%
}
\end{table}

\begin{table}[h!]
\centering
\caption{EDL fine-tuning hyperparameters for MiMo-7B-Base across the four evaluation datasets.}
\label{tab:hparams_mimo_7b}
\small
\resizebox{\textwidth}{!}{%
\begin{tabular}{lcccc}
\toprule
Hyperparameter & GSM8K & MATH & ChemistryQA & FormulaReasoning \\
\midrule
\multicolumn{5}{l}{\textit{Model \& LoRA}} \\
Model \texttt{max\_seq\_length}  & 2048 & 4096 & 4096 & 4096 \\
LoRA $r$ / $\alpha$ / dropout    & 16 / 32 / 0.0 & 16 / 32 / 0.0 & 16 / 32 / 0.0 & 16 / 32 / 0.0 \\
LoRA target modules              & Q,K,V,O,gate,up,down & Q,K,V,O,gate,up,down & Q,K,V,O,gate,up,down & Q,K,V,O,gate,up,down \\
LoRA bias                        & none & none & none & none \\
Precision                        & bf16 & bf16 & bf16 & bf16 \\
\midrule
\multicolumn{5}{l}{\textit{Optimizer (cosine schedule)}} \\
Optimizer                        & AdamW & AdamW & AdamW & AdamW \\
LR scheduler                     & cosine & cosine & cosine & cosine \\
Warmup ratio                     & 0.06 & 0.06 & 0.06 & 0.06 \\
Max grad norm                    & 1.0 & 1.0 & 1.0 & 1.0 \\
Gradient checkpointing           & enabled & enabled & enabled & enabled \\
\midrule
\multicolumn{5}{l}{\textit{EDL training}} \\
EDL initialization               & base & SFT ckpt (epoch=1) & SFT ckpt (epoch=1) & SFT ckpt (epoch=1) \\
\texttt{num\_train\_epochs}      & 2.5 & 0.3 & 0.3 & 0.5 \\
Per-device batch size            & 4 & 4 & 4 & 4 \\
Peak LR                          & $2\!\times\!10^{-4}$ & $5\!\times\!10^{-5}$ & $5\!\times\!10^{-5}$ & $5\!\times\!10^{-5}$ \\
\texttt{edl\_weight\_start\_step} & 0.8 & 0.0 & 0.0 & 0.0 \\
EDL $\beta$                      & 1.2 & 1.2 & 1.2 & 1.2 \\
EDL weight $\lmseweight$         & 5.0 & 5.0 & 5.0 & 5.0 \\
$N_{\text{train}}$               & 32 & 12 & 12 & 12 \\
Sampling \texttt{max\_tokens}    & 256 & 512 & 512 & 512 \\
Sampling temperature             & 1.0 & 0.5 & 0.5 & 1.0 \\
Sampling \texttt{top\_p}         & 0.9 & 0.9 & 0.9 & 0.9 \\
\midrule
\multicolumn{5}{l}{\textit{Evaluation}} \\
$N_{\text{eval}}$                & 50 & 50 & 50 & 50 \\
Temperature                      & 0.7 & 0.7 & 0.7 & 0.7 \\
\texttt{top\_p}                  & 0.9 & 0.9 & 0.9 & 0.9 \\
\texttt{max\_tokens}             & 1024 & 2048 & 2048 & 2048 \\
\midrule
\multicolumn{5}{p{1.3\linewidth}}{\footnotesize\emph{Note.} All EDL runs use a cosine LR schedule with \texttt{warmup\_ratio}~$=0.06$ over the entire \texttt{num\_train\_epochs}; the EDL loss is enabled starting at fraction \texttt{edl\_weight\_start\_step} of total steps. Where \emph{EDL initialization} is \emph{base}, the cosine schedule covers the full run and EDL begins partway through the decay; where \emph{EDL initialization} is \emph{SFT ckpt (epoch=1)}, the run is short and the cosine schedule covers only the EDL phase.} \\
\bottomrule
\end{tabular}%
}
\end{table}

\subsection{Training hyperparameters}
\label{sec:app:training_hyperparameters}

The full training and evaluation hyperparameters for each (model, dataset) cell are listed in \cref{tab:hparams_qwen3_8b}, \cref{tab:hparams_ministral3_8b}, and \cref{tab:hparams_mimo_7b}, covering Qwen3-8B-Base, Ministral-3-8B-Base, and MiMo-7B-Base respectively.
For MATH and GSM8K we use \texttt{math-verify}~\citep{kydlicek2025mathverify} for semantic clustering at both training and evaluation time. On ChemistryQA and FormulaReasoning, we employ GPT-OSS-20B~\citep{agarwal2025gpt} as an LLM-as-judge for semantic clustering at both training and evaluation time, with the system prompt shown in \cref{sec:app:prompt_llm_as_judge}.

\subsection{Hyperparameters of baselines}
\label{sec:app:baseline_hyperparameters}

\paragraph{BLoB} During training, we set \texttt{bayes\_samples} to 1, \texttt{bayes\_epsilon} to 0.05, \texttt{bayes\_beta} to 0.2, and \texttt{bayes\_kllr} to 0.1, with \texttt{bayes\_kl\_reweighting} enabled. During evaluation, we use 5 Bayesian samples. We refer readers to the official BLoB codebase for the precise definitions of these hyperparameters.

\paragraph{RCFT} During training, we configure the hyperparameters by setting \texttt{rcft\_alpha} to 0.5, \texttt{rcft\_calibration\_scale} to 0.001, \texttt{rcft\_num\_bins} to 10, and \texttt{rcft\_eps} to 0.000001. Please refer to the official RCFT codebase for detailed explanations of these variables.

\paragraph{LS} During training, we apply a label smoothing coefficient of $0.1$.

\paragraph{Dr.GRPO} We start the fine tuning process from the supervised checkpoint at epoch 1. To ensure a fair comparison, we set the total number of training steps for Dr.GRPO to exactly match the number of steps where the evidential loss is active during S-EDL training. We use a group size of 8. Following \citet{liu2025understanding}, we do not use the KL divergence penalty w.r.t. the reference model.

%---END SECTION---%

%---SECTION---%
\section{Extended experimental results}

\subsection{In-distribution evaluations across varying temperatures}
\label{sec:app:eval_across_temperatures}
We evaluate the fine-tuned models across a range of sampling temperatures from 0.4 to 1.0 to test the effectiveness of calibration performance at different temperatures. The experimental settings are identical to those detailed in \cref{sec:in_distribution_evaluation}, except for the varied temperature values, and we report the average performance metrics for the three models in \cref{fig:eval_across_temperatures}. S-EDL and Dr.GRPO maintain competitive semantic accuracy across the temperature spectrum. Furthermore, S-EDL generally achieves low calibration error across the evaluated temperatures. In addition, this experiment also indicates that temperature scaling techniques~\citep{lamb2026improving,shen2024thermometer} can be used together with our method to further boost the calibration performance.

\begin{figure}[th]
    \centering
    \begin{subfigure}[b]{\textwidth}
        \centering
        \includegraphics[width=\linewidth]{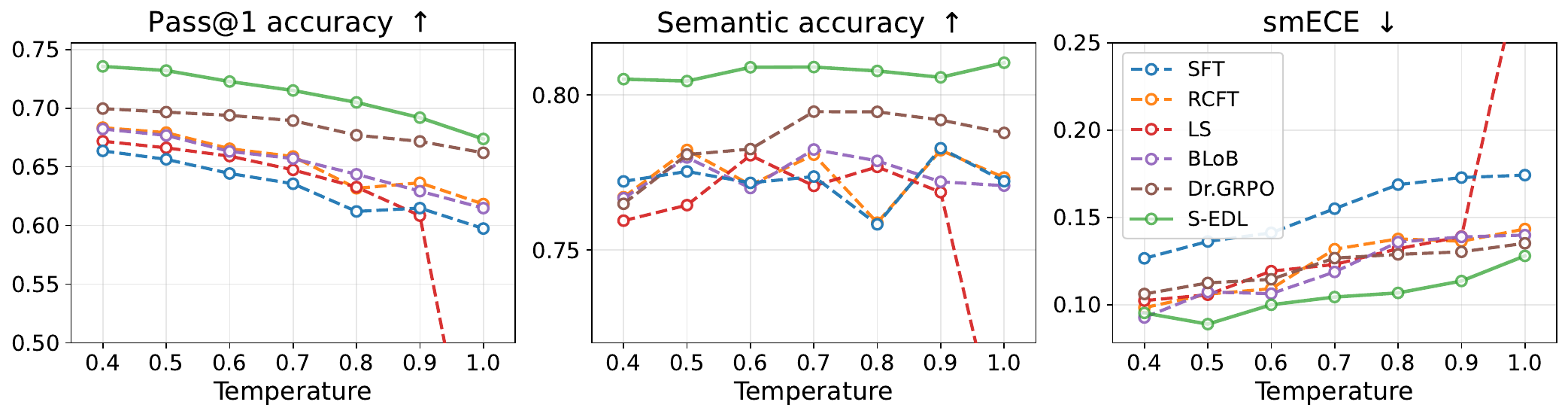}
        \caption{Fine-tuned model: Qwen3-8B-Base}
    \end{subfigure}
    \begin{subfigure}[b]{\textwidth}
        \centering
        \includegraphics[width=\linewidth]{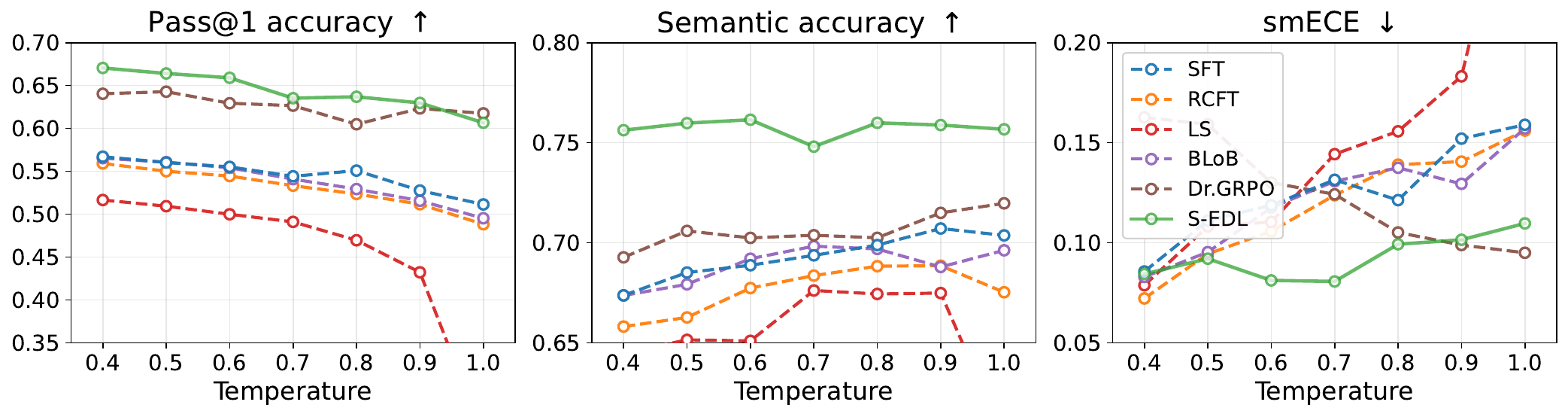}
        \caption{Fine-tuned model: Ministral3-8B-Base}
    \end{subfigure}
    \begin{subfigure}[b]{\textwidth}
        \centering
        \includegraphics[width=\linewidth]{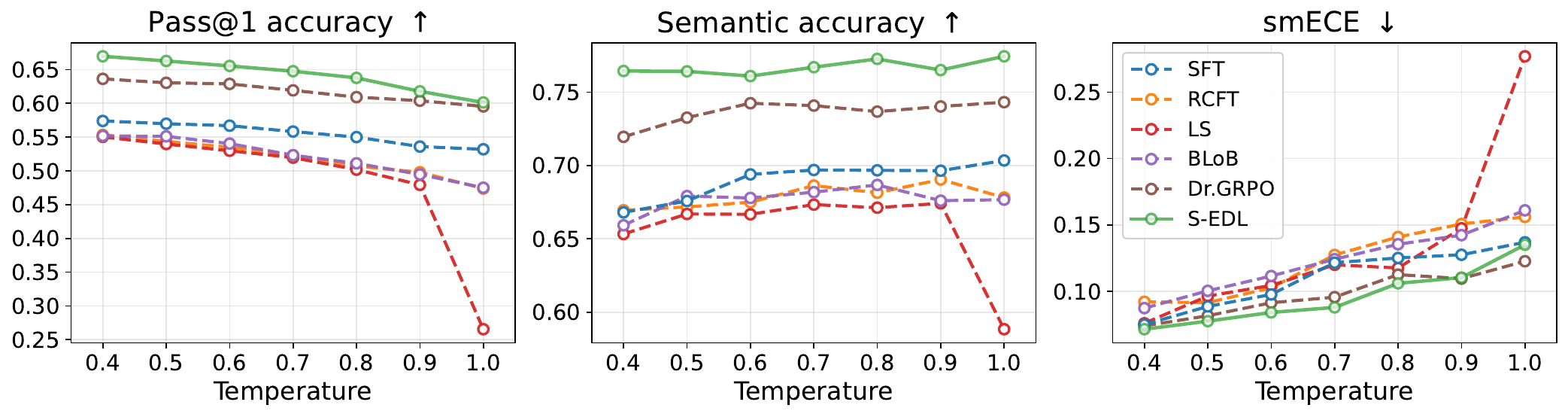}
        \caption{Fine-tuned model: MiMo-7B-Base}
    \end{subfigure}
    \caption{Average in-distribution performance across varying temperatures. Experiment settings are consistent with \cref{sec:in_distribution_evaluation}, except for the various temperatures. Note that certain data points, such as those for LS, may fall outside the plotted area because the y-axis limits are clamped to ensure better readability and visual clarity.}
    \label{fig:eval_across_temperatures}
\end{figure}

% \clearpage
% \newpage
\subsection{Visualization of the accuracy-calibration plane}
\label{sec:app:acc_calibration_plane}
\begin{figure}[t]
  \centering
  \begin{subfigure}[t]{0.32\linewidth}
    \centering
    \includegraphics[width=\linewidth]{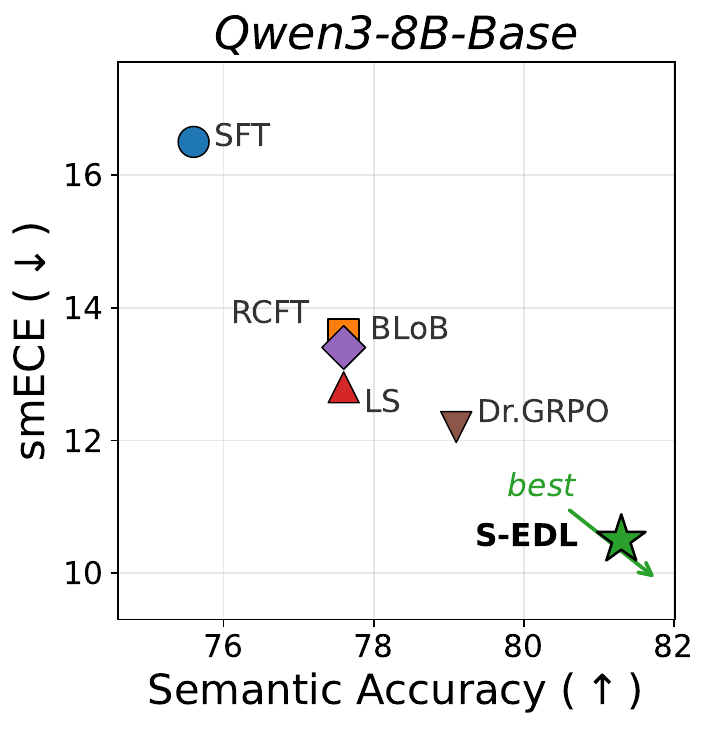}
    \caption{\textit{Qwen3-8B-Base}}
    \label{fig:acc_cal_qwen3}
  \end{subfigure}
  \hfill
  \begin{subfigure}[t]{0.32\linewidth}
    \centering
    \includegraphics[width=\linewidth]{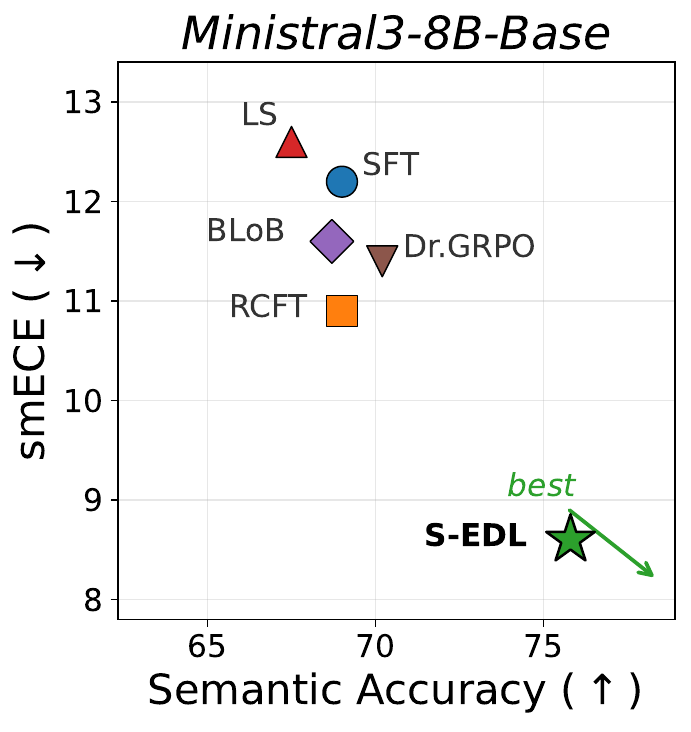}
    \caption{\textit{Ministral3-8B-Base}}
    \label{fig:acc_cal_ministral3}
  \end{subfigure}
  \hfill
  \begin{subfigure}[t]{0.32\linewidth}
    \centering
    \includegraphics[width=\linewidth]{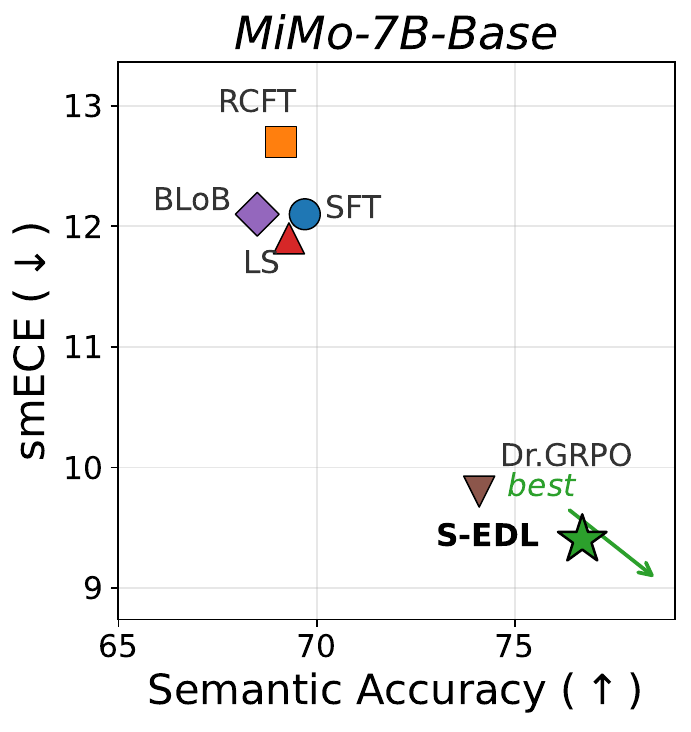}
    \caption{\textit{MiMo-7B-Base}}
    \label{fig:acc_cal_mimo}
  \end{subfigure}
  \caption{Accuracy--calibration plane, averaged across the four datasets. Bottom-right is the best corner (high S-Acc, low smECE).}
  \label{fig:accuracy_calibration}
\end{figure}

\cref{fig:accuracy_calibration} visualizes the averaged semantic accuracy and smECE plane across the four evaluated datasets. An optimal model should be situated in the bottom right corner of the plane, indicating both high accuracy and low calibration error. S-EDL consistently lies on the Pareto frontier of this plane across all three models.

\subsection{Extended per-token probability profiles}\label{sec:app:additional_per_token_prob}

\begin{figure}[t]
    \centering
    \begin{subfigure}[b]{\textwidth}
        \centering
        \includegraphics[width=\linewidth]{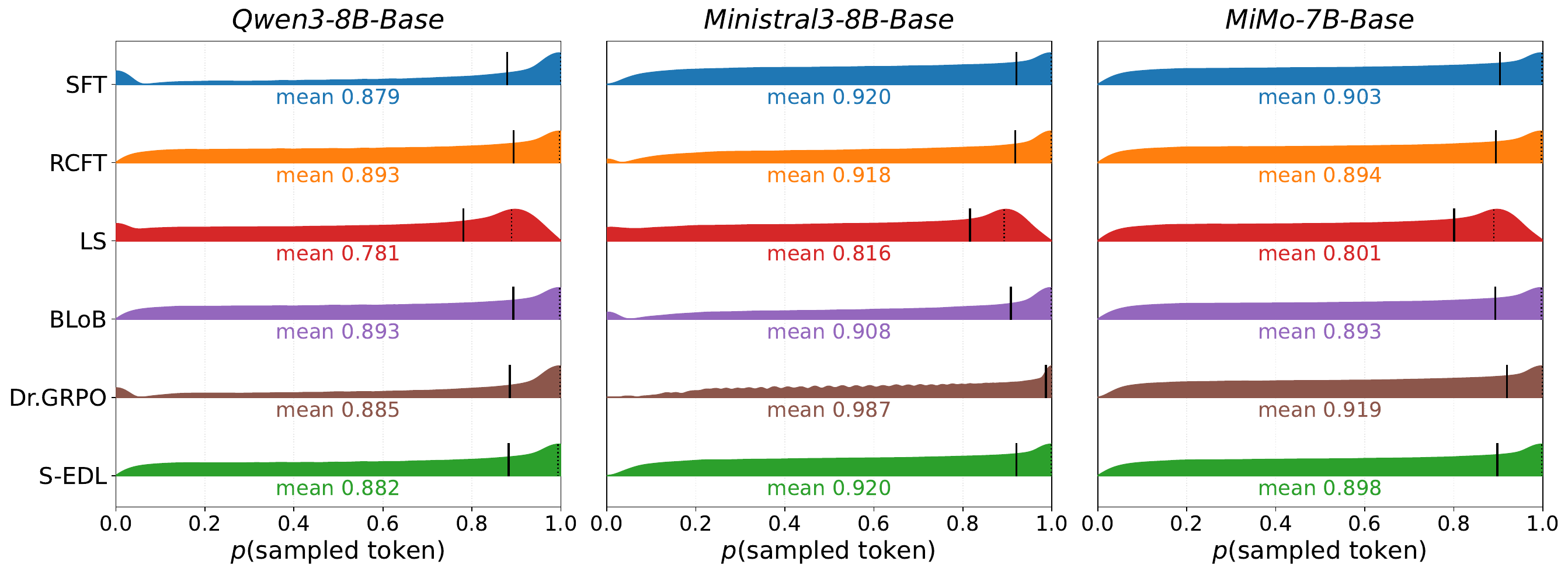}
        \caption{GSM8K}
        \label{fig:violin-gsm8k}
    \end{subfigure}
    \begin{subfigure}[b]{\textwidth}
        \centering
        \includegraphics[width=\linewidth]{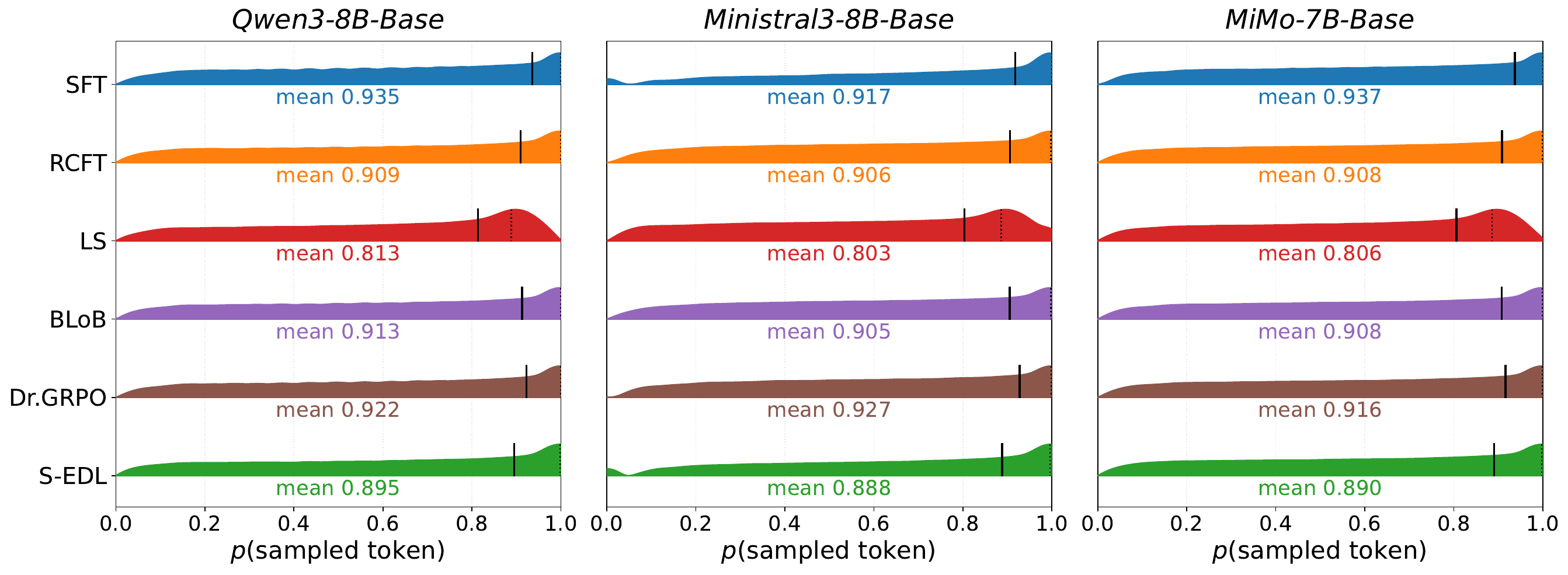}
        \caption{ChemistryQA}
        \label{fig:violin-chemistryqa}
    \end{subfigure}
    \begin{subfigure}[b]{\textwidth}
        \centering
        \includegraphics[width=\linewidth]{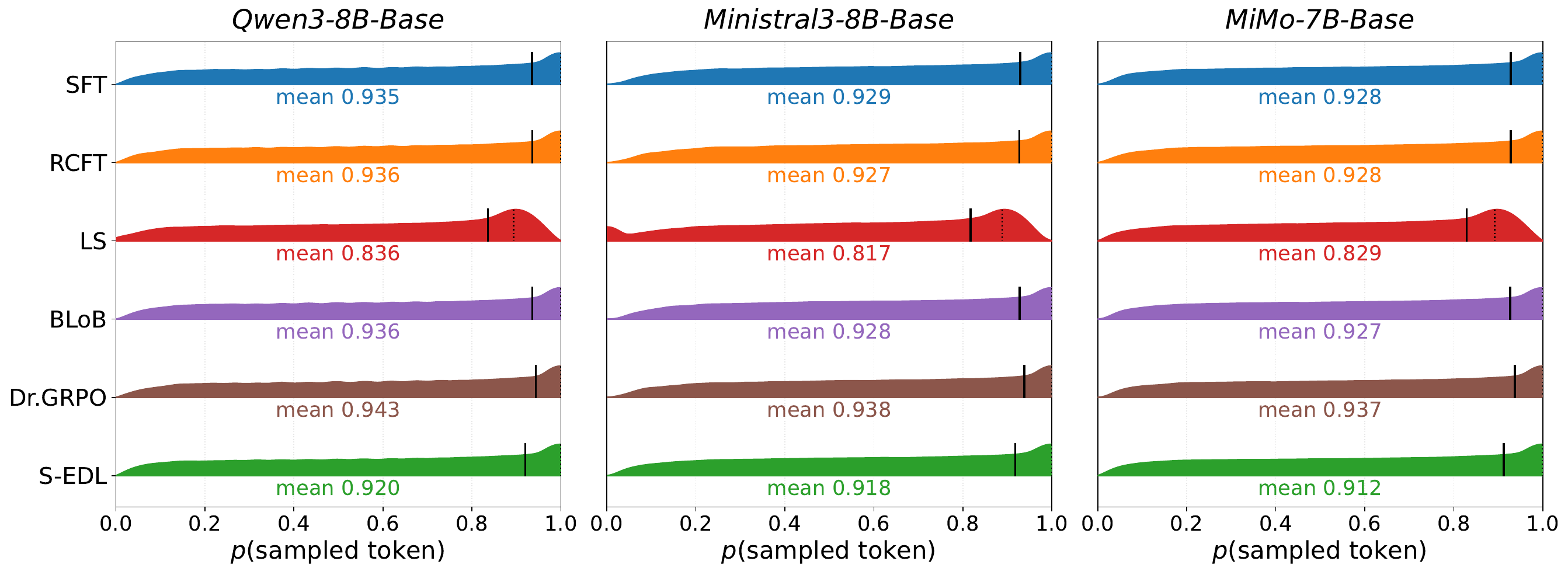}
        \caption{FormulaReasoning}
        \label{fig:violin-formulareasoning}
    \end{subfigure}
    \caption{Token probability distributions on GSM8K, ChemistryQA, and FormulaReasoning benchmarks.}
    \label{fig:token-prob-violin}
\end{figure}

We provide the extended per token probability profiles for the GSM8K, ChemistryQA, and FormulaReasoning datasets to supplement the main text analysis. \cref{fig:token-prob-violin} visualizes the token probability distributions across these three benchmarks. Consistent with the observations on the MATH dataset, the absolute difference between the S-EDL mean and the supervised fine tuning mean stays below 0.06 across all model and dataset pairs. This demonstrates that S-EDL achieves semantic calibration improvements without causing a drastic or systematic shift in the overall token level probability distribution. See \cref{fig:token-prob-violin} for the complete visual comparisons.

\clearpage
\subsection{Additional ablation study on $\beta$ and $\lmseweight$} \label{sec:app:additional_ablation_beta_weight}

\begin{figure}[h!]
    \centering
    \begin{subfigure}[b]{0.32\textwidth}
        \centering
        \includegraphics[width=\linewidth]{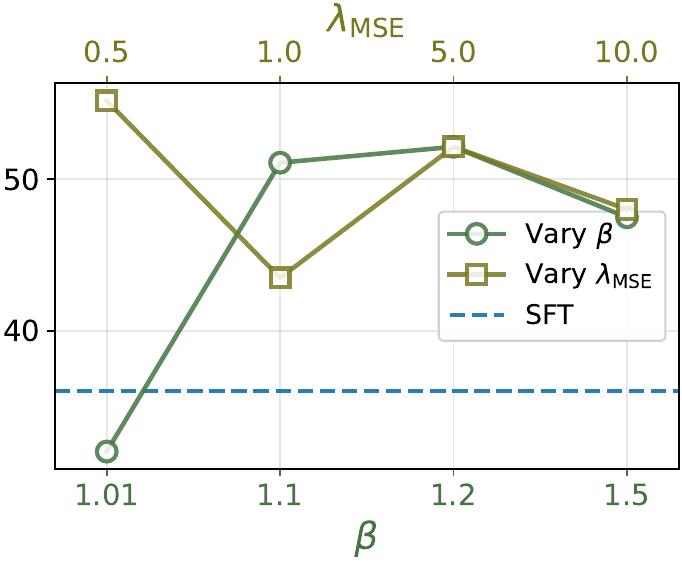}
        \caption{Pass@1 ($\uparrow$)}
    \end{subfigure}
    \begin{subfigure}[b]{0.32\textwidth}
        \centering
        \includegraphics[width=\linewidth]{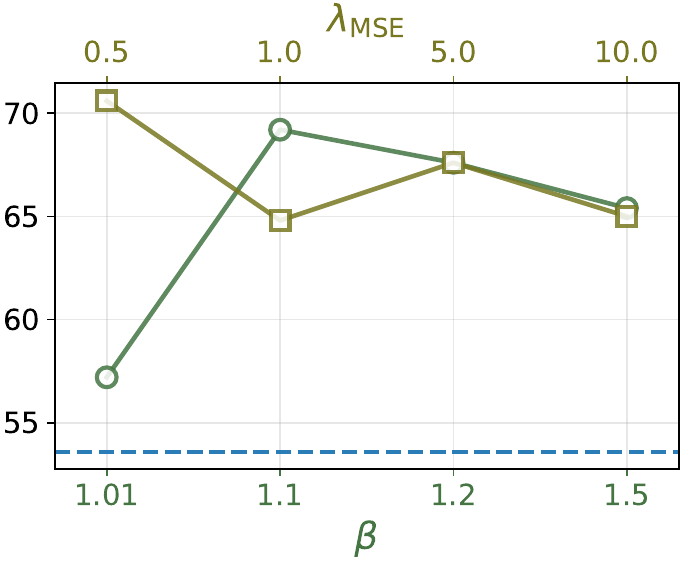}
        \caption{S-Acc. ($\uparrow$)}
    \end{subfigure}
    \begin{subfigure}[b]{0.32\textwidth}
        \centering
        \includegraphics[width=\linewidth]{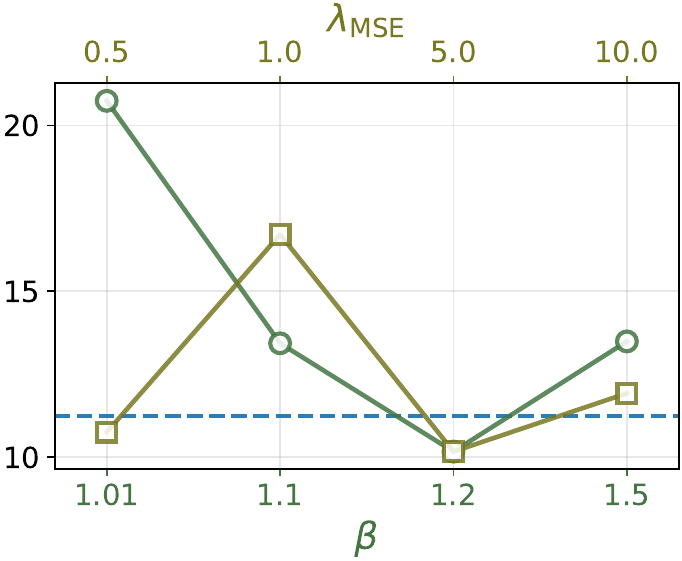}
        \caption{smECE ($\downarrow$)}
    \end{subfigure}
    \caption{Ablation study on $\beta$ and $\lambda_{\text{MSE}}$ with Ministral3-8B-Base trained on MATH and evaluated on MATH-500. The bottom $x$-axis varies $\beta$ with $\lambda_{\text{MSE}} = 5.0$ held fixed; the top $x$-axis varies $\lambda_{\text{MSE}}$ with $\beta = 1.2$ fixed. The default values $\beta=1.2$ and $\lmseweight=5.0$ demonstrate strong performance and are used throughout all experiments.}
    \label{fig:ablation_beta_weight_ministral3_8b}
\end{figure}

We present an additional ablation study on the hyperparameters $\beta$ and $\lmseweight$ using the Ministral3-8B-Base model trained on the MATH dataset and evaluated on the MATH-500 benchmark. As shown in \cref{fig:ablation_beta_weight_ministral3_8b}, the performance trends are consistent with those observed on the GSM8K dataset in the main text. Setting $\beta$ to a value close to 1, such as 1.01, severely weakens the evidential learning signal, which leads to a noticeable drop in accuracy and a spike in calibration error compared to the SFT baseline. In contrast, using the default parameters of $\beta=1.2$ and $\lmseweight=5.0$ provides a strong balance across all metrics, improving both Pass@1 and semantic accuracy while effectively lowering the smECE. This confirms that the selected default hyperparameters generalize well to harder reasoning tasks.

%---END SECTION---%

% \clearpage

\subsection{Prompts}\label{sec:app:prompts}

\subsubsection{System prompt of the LLM-as-judge} \label{sec:app:prompt_llm_as_judge}
\input{prompts/clustering}

\subsection{Examples of training prompts}
\label{sec:app:examples_training_prompts}

In the following, we provide examples of the training prompts used for fine tuning Qwen3-8B-Base. We intentionally do not enclose the chain of thought reasoning within \textless{}think\textgreater{} and \textless{}/think\textgreater{} blocks. We empirically observed that the base models have not been extensively trained on data containing these special tokens. Introducing these structural tokens during LoRA fine tuning causes the model to append random text after the \textbackslash{}boxed\{\} output and occasionally become trapped in repetitive generation loops.

\input{prompts/gsm8k_qwen3_8b}
\input{prompts/math_qwen3_8b}
\input{prompts/chemistryqa_qwen3_8b}
\input{prompts/formulareasoning_qwen3_8b}

%% file: prompts/clustering.tex
\begin{prompt}[System prompt of the LLM-as-judge for semantic clustering]
You are an answer-equivalence judge.

\medskip
\textbf{Task:}
\begin{itemize}
  \item You are given a Candidate Answer and a list of cluster representatives indexed as:\\
        0: \ldots, 1: \ldots, 2: \ldots
  \item Return the index of the cluster representative that is equivalent to the Candidate Answer.
  \item If none are equivalent, return $-1$.
\end{itemize}

\textbf{Equivalence rules:}
\begin{enumerate}
  \item Treat formatting differences as equivalent when numeric meaning matches:
        decimal/fraction/scientific/percentage/currency/unit notation.
  \item Unit omission is acceptable if numeric value matches.
  \item Allow rounding differences only when values agree to at least 3 significant digits.
  \item Reject sign flips and order-of-magnitude errors.
  \item For non-numeric answers, judge semantic equivalence of meaning, ignoring minor
        phrasing or word-order differences.
  \item For chemical equations:
  \begin{itemize}
    \item Arrow variants are equivalent ($\rightarrow$, ->, $\longrightarrow$, $\rightleftharpoons$).
    \item Ordering of species within each side is irrelevant (A + B $\equiv$ B + A).
    \item Subscript/superscript vs inline notation is equivalent
          ($\mathrm{N_2H_4} \equiv$ N2H4, $\mathrm{ClO_3^-} \equiv$ ClO3-).
    \item Stoichiometrically scaled equations (all coefficients multiplied by the same
          positive integer) are equivalent.
    \item State symbol omission is acceptable, but conflicting state symbols are NOT equivalent
          (e.g., H2O matches H2O(l), but H2O(g) does NOT match H2O(l)).
  \end{itemize}
\end{enumerate}

\textbf{Few-shot examples:}

\textit{Candidate Answer:} 3.50 mol\\
\textit{Cluster representatives:}\\
0: 3500 mmol\\
1: 3.5 moles\\
2: 0.35 mol\\
\textit{Matching index:} 1

\medskip
\textit{Candidate Answer:} $\frac{84}{2}$\\
\textit{Cluster representatives:}\\
0: 41\\
1: 42\\
2: $-42$\\
\textit{Matching index:} 1

\medskip
\textit{Candidate Answer:} 42\\
\textit{Cluster representatives:}\\
0: 42 J\\
1: 43\\
\textit{Matching index:} 0

\medskip
\textit{Candidate Answer:} the speed is about 10.3 m/s\textsuperscript{2}\\
\textit{Cluster representatives:}\\
0: 10.3 m/s\\
1: $-10.3$ m/s\textsuperscript{2}\\
\textit{Matching index:} $-1$

\medskip
\textit{Candidate Answer:} 4 $\mathrm{ClO_3^-}$(aq) + 3 $\mathrm{N_2H_4}$(g) $\rightarrow$
  4 $\mathrm{Cl^-}$(aq) + 6 $\mathrm{H_2O}$(l) + 6 NO(g)\\
\textit{Cluster representatives:}\\
0: 3 $\mathrm{N_2H_4}$(g) + 4 $\mathrm{ClO_3^-}$(aq) -> 6 NO(g)
   + 6 $\mathrm{H_2O}$(g) + 4 $\mathrm{Cl^-}$(aq)\\
1: 3 $\mathrm{N_2H_4}$(g) + 4 $\mathrm{ClO_3^-}$(aq) $\rightarrow$
   6 NO(g) + 6 $\mathrm{H_2O}$(l) + 4 $\mathrm{Cl^-}$(aq)\\
\textit{Matching index:} 1

\medskip
\textbf{Output format:}
\begin{itemize}
  \item Final line must be exactly:\\
        \texttt{Matching index: <integer>}
  \item Use $-1$ when none match.
\end{itemize}
\end{prompt}

%% file: prompts/gsm8k_qwen3_8b.tex
\begin{prompt}[An example of training prompts on GSM8K dataset]
\textless{}\textbar{}im\_start\textbar{}\textgreater{}system\\
Solve the following problem. First, reason through the problem step by step after 'Reasoning:'. Then, provide your final answer after 'My answer is:' wrapped in \textbackslash{}boxed\{\}.\textless{}\textbar{}im\_end\textbar{}\textgreater{}\\
\textless{}\textbar{}im\_start\textbar{}\textgreater{}user\\
Question: Natalia sold clips to 48 of her friends in April, and then she sold half as many clips in May. How many clips did Natalia sell altogether in April and May?\\
Answer:\textless{}\textbar{}im\_end\textbar{}\textgreater{}\\
\textless{}\textbar{}im\_start\textbar{}\textgreater{}assistant\\
Reasoning: Natalia sold 48/2 = \textless{}\textless{}48/2=24\textgreater{}\textgreater{}24 clips in May.\\
Natalia sold 48+24 = \textless{}\textless{}48+24=72\textgreater{}\textgreater{}72 clips altogether in April and May.\\
My answer is: \textbackslash{}boxed\{72\}\textless{}\textbar{}endoftext\textbar{}\textgreater{}
\end{prompt}

%% file: prompts/math_qwen3_8b.tex
\begin{prompt}[An example of training prompts on MATH dataset]
\textless{}\textbar{}im\_start\textbar{}\textgreater{}system\\
Solve the following problem. First, reason through the problem step by step after 'Reasoning:'. Then, provide your final answer after 'My answer is:' wrapped in \textbackslash{}boxed\{\}.\textless{}\textbar{}im\_end\textbar{}\textgreater{}\\
\textless{}\textbar{}im\_start\textbar{}\textgreater{}user\\
Question: How many 4-letter words with at least one consonant can be constructed from the letters \$A\$, \$B\$, \$C\$, \$D\$, and \$E\$? (Note that \$B\$, \$C\$, and \$D\$ are consonants, any word is valid, not just English language words, and letters may be used more than once.)\\
Answer:\textless{}\textbar{}im\_end\textbar{}\textgreater{}\\
\textless{}\textbar{}im\_start\textbar{}\textgreater{}assistant\\
Reasoning: First we count the number of all 4-letter words with no restrictions on the word. Then we count the number of 4-letter words with no consonants. We then subtract to get the answer.\\
\\
Each letter of a word must be one of \$A\$, \$B\$, \$C\$, \$D\$, or \$E\$, so the number of 4-letter words with no restrictions on the word is \$5\textbackslash{}times 5\textbackslash{}times 5\textbackslash{}times 5=625\$. Each letter of a word with no consonant must be one of \$A\$ or \$E\$. So the number of all 4-letter words with no consonants is \$2\textbackslash{}times 2\textbackslash{}times 2\textbackslash{}times 2=16\$. Therefore, the number of 4-letter words with at least one consonant is \$625-16=609\$.\\
My answer is: \textbackslash{}boxed\{609\}\textless{}\textbar{}endoftext\textbar{}\textgreater{}
\end{prompt}

%% file: prompts/chemistryqa_qwen3_8b.tex
\begin{prompt}[An example of training prompts on ChemistryQA dataset]
\textless{}\textbar{}im\_start\textbar{}\textgreater{}system\\
Solve the following problem. First, reason through the problem step by step after 'Reasoning:'. Then, provide your final answer after 'My answer is:' wrapped in \textbackslash{}boxed\{\}.\textless{}\textbar{}im\_end\textbar{}\textgreater{}\\
\textless{}\textbar{}im\_start\textbar{}\textgreater{}user\\
Question: A mixture of \#"N"\_2\#, \#"O"\_2\#. and \#"CO"\_2\# has a total pressure of \#"740 torr"\# In this mixture, the partial pressure of \#"N"\_2\# is \#"120 torr"\# and the partial pressure of \#"O"\_2\# is \#"400 torr"\#. What is the partial pressure of \#"CO"\_2\# ?\\
~\\
~\\
Answer:\textless{}\textbar{}im\_end\textbar{}\textgreater{}\\
\textless{}\textbar{}im\_start\textbar{}\textgreater{}assistant\\
Reasoning: The total pressure equals the sum of the partial pressures:\\
~\\
\textbackslash{}[\\`
P\_\{\textbackslash{}text\{total\}\} = P\_\{\textbackslash{}mathrm\{N\_2\}\} + P\_\{\textbackslash{}mathrm\{O\_2\}\} + P\_\{\textbackslash{}mathrm\{CO\_2\}\}.\\
\textbackslash{}]\\
~\\
Given \textbackslash{}(P\_\{\textbackslash{}text\{total\}\} = 740\textbackslash{} \textbackslash{}text\{torr\},\textbackslash{}; P\_\{\textbackslash{}mathrm\{N\_2\}\} = 120\textbackslash{} \textbackslash{}text\{torr\},\textbackslash{}; P\_\{\textbackslash{}mathrm\{O\_2\}\} = 400\textbackslash{} \textbackslash{}text\{torr\}\textbackslash{}),\\
~\\
\textbackslash{}[\\`
P\_\{\textbackslash{}mathrm\{CO\_2\}\} = 740 - (120 + 400) = 740 - 520 = 220\textbackslash{} \textbackslash{}text\{torr\}.\\
\textbackslash{}]\\
My answer is: \textbackslash{}boxed\{220\textbackslash{} \textbackslash{}text\{torr\}\}\textless{}\textbar{}endoftext\textbar{}\textgreater{}
    
\end{prompt}

%% file: prompts/formulareasoning_qwen3_8b.tex
\begin{prompt}[An example of training prompts on FormulaReasoning dataset]
\textless{}\textbar{}im\_start\textbar{}\textgreater{}system\\
Solve the following problem. First, reason through the problem step by step after 'Reasoning:'. Then, provide your final answer with units after 'My answer is:' wrapped in \textbackslash{}boxed\{\}.\textless{}\textbar{}im\_end\textbar{}\textgreater{}\\
\textless{}\textbar{}im\_start\textbar{}\textgreater{}user\\
Question: Rational classification and utilization of waste can protect the environment and turn waste into valuables. Under certain conditions, 1 ton of classified waste can "extract" [X] of fuel oil, where the calorific value of the fuel oil is 3.0$\times$10\textasciicircum{}7 J/kg. Question: [c\_water=4.2$\times$10\textasciicircum{}3 J/(kg$\cdot$$^\circ$C)] How much heat energy is released when the fuel oil is completely burned? The specific heat capacity of water is 4.2$\times$10\textasciicircum{}3 J/(kg$\cdot$$^\circ$C). If we already know that the answer to the above question is 4.2$\times$10\textasciicircum{}9 J, then what should be the content corresponding to [X] in the question?\\
Answer:\textless{}\textbar{}im\_end\textbar{}\textgreater{}\\
\textless{}\textbar{}im\_start\textbar{}\textgreater{}assistant\\
Reasoning: The heat energy released when fuel oil is burned is calculated using \$ Q = m \textbackslash{}cdot \textbackslash{}text\{calorific value\} \$. Rearranging for mass \$ m \$, we get \$ m = \textbackslash{}frac\{Q\}\{\textbackslash{}text\{calorific value\}\} = \textbackslash{}frac\{4.2 \textbackslash{}times 10\textasciicircum{}9 \textbackslash{}, \textbackslash{}text\{J\}\}\{3.0 \textbackslash{}times 10\textasciicircum{}7 \textbackslash{}, \textbackslash{}text\{J/kg\}\} = 140 \textbackslash{}, \textbackslash{}text\{kg\} \$. Since this mass corresponds to the fuel oil extracted from 1 ton (1000 kg) of waste, \$[X]\$ must be 140 kg.\\
My answer is: \textbackslash{}boxed\{140 kg\}\textless{}\textbar{}endoftext\textbar{}\textgreater{}
\end{prompt}

%% file: checklist.tex
\section*{NeurIPS Paper Checklist}

\begin{enumerate}

\item {\bf Claims}
    \item[] Question: Do the main claims made in the abstract and introduction accurately reflect the paper's contributions and scope?
    \item[] Answer: \answerYes{} % Replace by \answerYes{}, \answerNo{}, or \answerNA{}.
    \item[] Justification: The abstract and introduction state three contributions: (i) extending EDL from fixed-choice to open-ended semantic calibration, (ii) a novel parametrization that derives a differentiable evidence vector from sequence likelihoods and theoretically unifies Semantic Entropy under the evidential framework, and (iii) empirical gains compared to baselines. These claims are directly supported by the methodology in \cref{sec:method}, the theoretical results in \cref{sec:gradient}, and the experimental results in \cref{tab:in_distribution_acc_vs_smece,tab:cross_domain} and \cref{sec:experiments}.

\item {\bf Limitations}
    \item[] Question: Does the paper discuss the limitations of the work performed by the authors?
    \item[] Answer: \answerYes{} % Replace by \answerYes{}, \answerNo{}, or \answerNA{}.
    \item[] Justification: \cref{sec:conclusion} contains a dedicated ``Limitations'' paragraph.

\item {\bf Theory assumptions and proofs}
    \item[] Question: For each theoretical result, does the paper provide the full set of assumptions and a complete (and correct) proof?
    \item[] Answer: \answerYes{} % Replace by \answerYes{}, \answerNo{}, or \answerNA{}.
    \item[] Justification: The paper presents \cref{thm:consistency}, \cref{thm:ll_gradient}, \cref{thm:sign_of_grad}, and \cref{thm:conservation}, each with explicitly stated assumptions (e.g., $\beta > 1$ for \cref{thm:sign_of_grad}; i.i.d. Monte Carlo sampling for \cref{thm:consistency}) and full proofs provided in \cref{sec:app:proof_consistency,sec:app:proof_ll_gradient,sec:app:proof_sign_of_grad,sec:app:proof_conservation}. All theoretical statements are numbered and cross-referenced from the main text.

    \item {\bf Experimental result reproducibility}
    \item[] Question: Does the paper fully disclose all the information needed to reproduce the main experimental results of the paper to the extent that it affects the main claims and/or conclusions of the paper (regardless of whether the code and data are provided or not)?
    \item[] Answer: \answerYes{} % Replace by \answerYes{}, \answerNo{}, or \answerNA{}.
    \item[] Justification: \cref{sec:exp_setup} specifies the three base models (Qwen3-8B-Base, Ministral3-8B-Base, MiMo-7B-Base), LoRA configuration (rank 16, $\alpha$=32), evidential hyperparameters ($\beta=1.2$, $\lmseweight=5.0$), per-dataset $N_{\text{train}}$ values, and the evaluation protocol ($\texttt{top\_p}=0.9$, per-model temperatures, $N_{\text{eval}}=50$). \cref{sec:app:dataset_construction} details dataset construction, \cref{sec:app:training_hyperparameters} lists full optimization hyperparameters, and \cref{sec:app:baseline_hyperparameters} documents baseline configurations. The find-first clustering algorithm is fully specified in \cref{alg:find-first}, and example training prompts are given in \cref{sec:app:prompts}.

\item {\bf Open access to data and code}
    \item[] Question: Does the paper provide open access to the data and code, with sufficient instructions to faithfully reproduce the main experimental results, as described in supplemental material?
    \item[] Answer: \answerYes{} % Replace by \answerYes{}, \answerNo{}, or \answerNA{}.
    \item[] Justification: We will release the code repository upon acceptance or upon requests from reviewers, including training scripts, the find-first clustering implementation, evaluation pipelines, and configuration files. The four datasets (GSM8K, MATH, ChemistryQA, FormulaReasoning) are publicly available; we provide preprocessing scripts to reproduce the filtered training splits described in \cref{sec:app:dataset_construction}.

\item {\bf Experimental setting/details}
    \item[] Question: Does the paper specify all the training and test details (e.g., data splits, hyperparameters, how they were chosen, type of optimizer) necessary to understand the results?
    \item[] Answer: \answerYes{} % Replace by \answerYes{}, \answerNo{}, or \answerNA{}.
    \item[] Justification: \cref{sec:exp_setup} documents data splits, base models, LoRA rank, evidential hyperparameters, the per-model decoding temperatures (chosen on a held-out validation split to match per-model entropy), and $N_{\text{eval}}$. \cref{sec:app:dataset_construction} specifies how each dataset's training split was constructed, \cref{sec:app:training_hyperparameters} lists optimizer, learning rate, batch size, and other training details, and \cref{sec:app:baseline_hyperparameters} documents the hyperparameters used for every baseline (RCFT, LS, BLoB, Dr.GRPO).

\item {\bf Experiment statistical significance}
    \item[] Question: Does the paper report error bars suitably and correctly defined or other appropriate information about the statistical significance of the experiments?
    \item[] Answer: \answerNo{} % Replace by \answerYes{}, \answerNo{}, or \answerNA{}.
    \item[] Justification: The main calibration and accuracy results in \cref{tab:in_distribution_acc_vs_smece} and \cref{tab:cross_domain} report point estimates from a single training run per (model, dataset, method) configuration without error bars, due to compute constraints associated with fine-tuning three 7–8B parameter models across four datasets and six methods. Variance information is reported where available (\cref{fig:clustering_complexity} shows standard deviations of verifier-call counts; \cref{fig:token_probability} and \cref{fig:token-prob-violin} visualize full token-probability distributions).

\item {\bf Experiments compute resources}
    \item[] Question: For each experiment, does the paper provide sufficient information on the computer resources (type of compute workers, memory, time of execution) needed to reproduce the experiments?
    \item[] Answer: \answerYes{} % Replace by \answerYes{}, \answerNo{}, or \answerNA{}.
    \item[] Justification: All fine-tuning experiments were conducted on NVIDIA H100 / H200 GPUs. Each run uses LoRA adapters (rank 16). Per-run training time is approximately 1-12 hours, depending on the fine-tuning method. Memory constraints determined $N_{\text{train}}$ per dataset (see \cref{sec:exp_setup}). Total compute including baselines, ablations, and preliminary runs is approximately 2,500 GPU-hours.
    
\item {\bf Code of ethics}
    \item[] Question: Does the research conducted in the paper conform, in every respect, with the NeurIPS Code of Ethics \url{https://neurips.cc/public/EthicsGuidelines}?
    \item[] Answer: \answerYes{} % Replace by \answerYes{}, \answerNo{}, or \answerNA{}.
    \item[] Justification: The research uses publicly available datasets and open-weight base models, does not involve human subjects, and aims to improve the trustworthiness of LLM outputs by reducing miscalibration. We have reviewed the NeurIPS Code of Ethics and the work conforms with it in all respects.

\item {\bf Broader impacts}
    \item[] Question: Does the paper discuss both potential positive societal impacts and negative societal impacts of the work performed?
    \item[] Answer: \answerYes{} % Replace by \answerYes{}, \answerNo{}, or \answerNA{}.
    \item[] Justification: Improving semantic calibration helps users and downstream systems decide when to trust an LLM's open-ended output, which is broadly beneficial in high-stakes settings such as scientific reasoning and decision support. Potential negative impacts are limited: S-EDL is a fine-tuning objective rather than a capability-extending technique, and a better-calibrated model could in principle make over-trusted but plausible, looking incorrect outputs less obvious, we mitigate this concern by reporting both accuracy and calibration error so practitioners can audit both dimensions.
    
\item {\bf Safeguards}
    \item[] Question: Does the paper describe safeguards that have been put in place for responsible release of data or models that have a high risk for misuse (e.g., pre-trained language models, image generators, or scraped datasets)?
    \item[] Answer: \answerNA{} % Replace by \answerYes{}, \answerNo{}, or \answerNA{}.
    \item[] Justification: The paper does not release new pre-trained models, image generators, or scraped datasets that pose a high risk of misuse. The contribution is a fine-tuning method evaluated on existing benchmarks.

\item {\bf Licenses for existing assets}
    \item[] Question: Are the creators or original owners of assets (e.g., code, data, models), used in the paper, properly credited and are the license and terms of use explicitly mentioned and properly respected?
    \item[] Answer: \answerYes{} % Replace by \answerYes{}, \answerNo{}, or \answerNA{}.
    \item[] Justification: All datasets (GSM8K, MATH, ChemistryQA, FormulaReasoning) and base models (Qwen3-8B-Base, Ministral3-8B-Base, MiMo-7B-Base) are properly cited. External tools math-verify and rapidfuzz are credited with version numbers. 

\item {\bf New assets}
    \item[] Question: Are new assets introduced in the paper well documented and is the documentation provided alongside the assets?
    \item[] Answer: \answerNA{} % Replace by \answerYes{}, \answerNo{}, or \answerNA{}.
    \item[] Justification: The paper does not release new assets.

\item {\bf Crowdsourcing and research with human subjects}
    \item[] Question: For crowdsourcing experiments and research with human subjects, does the paper include the full text of instructions given to participants and screenshots, if applicable, as well as details about compensation (if any)? 
    \item[] Answer: \answerNA{} % Replace by \answerYes{}, \answerNo{}, or \answerNA{}.
    \item[] Justification: The paper does not involve crowdsourcing nor research with human subjects.

\item {\bf Institutional review board (IRB) approvals or equivalent for research with human subjects}
    \item[] Question: Does the paper describe potential risks incurred by study participants, whether such risks were disclosed to the subjects, and whether Institutional Review Board (IRB) approvals (or an equivalent approval/review based on the requirements of your country or institution) were obtained?
    \item[] Answer: \answerNA{} % Replace by \answerYes{}, \answerNo{}, or \answerNA{}.
    \item[] Justification: The paper does not involve crowdsourcing nor research with human subjects.

\item {\bf Declaration of LLM usage}
    \item[] Question: Does the paper describe the usage of LLMs if it is an important, original, or non-standard component of the core methods in this research? Note that if the LLM is used only for writing, editing, or formatting purposes and does \emph{not} impact the core methodology, scientific rigor, or originality of the research, declaration is not required.
    %this research? 
    \item[] Answer: \answerYes{} % Replace by \answerYes{}, \answerNo{}, or \answerNA{}.
    \item[] Justification: LLMs are used as components of the methodology beyond writing assistance: (i) GPT-OSS-120B is used to generate chain-of-thought solutions for the ChemistryQA and FormulaReasoning training splits, with filtering to retain only correct generations (\cref{sec:app:dataset_construction}); (ii) for general (non-mathematical) semantic clustering, an LLM-as-judge is used as the equivalence oracle MATCH in \cref{alg:find-first}, with the system prompt provided verbatim in \cref{sec:app:prompt_llm_as_judge}. The core S-EDL objective itself does not depend on LLM-judge supervision; it operates on the model's own sequence likelihoods.

\end{enumerate}